\documentclass{scrartcl}

\usepackage{amsmath, amsthm, amssymb}
\usepackage{mathrsfs}
\usepackage{url}
\usepackage{breqn}
\usepackage{amssymb,amsfonts}
\usepackage{color}
\usepackage{shuffle}
\usepackage[usenames,dvipsnames]{xcolor}
\usepackage{graphicx}
\usepackage[font={small,it}]{caption}
\usepackage[flushleft]{threeparttable}
\usepackage[linecolor=white,backgroundcolor=white,bordercolor=white,textsize=tiny]{todonotes}
\usepackage{tikz}
\usepackage{pgfplots}
\pgfplotsset{compat=1.13}
\usetikzlibrary{arrows,shapes,positioning}
\usetikzlibrary{arrows.meta}
\usetikzlibrary{decorations.markings}

\usepackage{ytableau}

\usepackage{silence}
\usepackage[top=3cm, bottom=3cm, left=2.5cm, right=3cm]{geometry}

\newtheorem{theorem}{Theorem}
\theoremstyle{plain}

\newtheorem{conjecture}[theorem]{Conjecture}

\newtheorem{example}[theorem]{Example}

\newtheorem{lemma}[theorem]{Lemma}

\newtheorem{proposition}[theorem]{Proposition}
\newtheorem{remark}[theorem]{Remark}

\theoremstyle{definition} 

\newtheorem{definition}[theorem]{Definition}

\newcommand{\R}{\mathbb{R}}
\newcommand{\C}{\mathbb{C}}
\newcommand{\N}{\mathbb{N}}

\newcommand{\spann}{\operatorname{span}}

\newcommand{\ds}{d} % dimension of the signal
\newcommand{\TC}{T((\R^\ds))} % concat
\newcommand{\TS}{T(\R^\ds)} % shuffle

\newcommand{\sign}{\operatorname{sign}}

\def\word#1{{\color{blue}\mathbf{#1}}}

\usepackage[colorlinks,backref]{hyperref}
\begin{document}

\title{Invariants of multidimensional time series based on their iterated-integral signature}
%\author{Joscha Diehl\thanks{MPI for Mathematics in the Sciences, Leipzig}, Jeremy Reizenstein\thanks{Centre for Complexity Science, University of Warwick. Supported by the Engineering and Physical Sciences Research Council}}
\author{Joscha Diehl\thanks{MPI for Mathematics in the Sciences, Leipzig}, Jeremy Reizenstein\thanks{Centre for Complexity Science, University of Warwick. Supported by the Engineering and Physical Sciences Research Council}}
%\date{}
\maketitle

\begin{abstract}
  %We show how some of the ``integral invariants'' found in the literature are subsumed in our general setting.
We introduce a novel class of features for multidimensional time series that are invariant with respect to transformations of the ambient space. The general linear group, the group of rotations and the group of permutations of the axes are considered. The starting point for their construction is Chen's iterated-integral signature.
\end{abstract}

\tableofcontents

\section{Introduction}

%{\color{orange
%Data that can be represented as a curve in two dimensional Euclidean space appears in many
%areas of pattern analysis. The boundary (or silhouette) of an object in an image or video
%recording can be stored as a (closed) two dimensional curve. This has been used for example
%for object recognition of everyday objects [31], the mapping of (aerial) photographs to terrain
%maps [30, 36] and to solving jigsaw puzzles [17]. In video recordings, the path of persons in the
%field of vision can be analyzed to detect anomalous behavior [23, 43] or predict future moves [5].
%Hand gestures also naturally describe trajectories [34]. On a larger scale, the travel path, usually
%of a person or vehicel, in its environment is easily recovered from time stamped location data,
%given for example either from GPS recordings or derived from mobile connectivity logs. See the
%monograph [42] for an overview of applications in this context. In online (stroke-based) character
%recognition, the input stream is usually the trajectory of pen-movements, see for example the
%survey [33].
%In some of these applications, there is no fixed coordinate system in which to orient the data, so
%it becomes important to select features of the data that are invariant to rotation of the input.
%In the example of online character recognition, thinking of multiuser tablets, the angle from
%which the device is used does not have to be fixed. This fact has to be taken into account when
%extracting features for the recognition task.}

%{\color{orange}
%- GPS data\\
%- video of people moving\\
%- EEG data}

The analysis of multidimensional time series is a standard problem in data science.
Usually, as a first step, features of a time series must be extracted that are
(in some sense) robust and that characterize the time series.
%
%(The most well known procedure to obtain such features might be via spectral analysis,
%that is computing the Fourier series of the signal.)
%
In many applications the features should additionally be invariant to a particular group
acting on the data.
%
%In character recognition on tablets for example,
%the angle from which the device is operated is usually not fixed.
%This leads to the requirement of rotation invariant features.
In Human Activity Recognition for example,
the orientation of the measuring device is often unknown.
This leads to the requirement of rotation invariant features \cite{bib:MA2017}.
In EEG analysis, invariants to the general linear group are beneficial \cite{bib:EMZMN2012}.
In other applications, the labeling of coordinates is arbitrary, which leads to
%for example if we have observations of some system from a set of random or unlabeled sensors.
permutation invariants.
%\todonotes{ jeremy, do you have something good in mind here?
%only thing i can think of is this persistent homology stuff, which is a bit artificial (since the signature is probably not useful there).}

%This fact has to be taken into account when extracting features for the recognition task.
%When thinking about EEG data from one single subject, the position of electrodes on the scalp
%varies slightly from experiment to experiment.
%So (some) invariance with respect to

As any time series in discrete time can, via linear interpolation, be thought of as a multidimensional curve,
one is naturally lead to the search of invariants of curves.
Invariant features, of (mostly) two-dimensional curves
have been treated using various approaches.
%Rotation invariant features
%of two-dimensional curves
%have been treated using various approaches.
%
Among the techniques are
Fourier series (of closed curves) \cite{bibGranlund1972,bibZahnRoskies1972,bibKuhlGiardina1982},
wavelets \cite{bibChuangKuo1996},
curvature based methods \cite{bibMokhtarianMackworth1986,bib:COSTH1998} %\todonotes{ this Olver stuff is only $2$ (maybe $3$) dimensional! }
and integral invariants \cite{bibManayCremers2006,bib:FKK2010}.
%Closest to our setting is probably the work
%\cite{bib:FKK2010}.

%Let us also mention the works
%on codons \cite{bibHoffmanRichards1982}
%the primal curvature sketch \cite{bibAsadaBrady1986}
%and on Freeman chains \cite{bibFreeman1974},
%%which all decompose curves into a sequence
%which are usually not rotation invariant.
% The latter 
% Simple con-tour functions such as chain code, centroidcontour distance (CCD) and RS curve just admit limited invariants (Kin-dratenko 2003).

%and
%invariant to translation, rotation, permutation, ..

%In video recordings, the path of persons in the
%field of vision can be analyzed to detect anomalous behavior [23, 43] or predict future moves [5].
%Hand gestures also naturally describe trajectories [34]. On a larger scale, the travel path, usually
%of a person or vehicel, in its environment is easily recovered from time stamped location data,
%given for example either from GPS recordings or derived from mobile connectivity logs. See the
%monograph [42] for an overview of applications in this context. In online (stroke-based) character
%recognition, the input stream is usually the trajectory of pen-movements, see for example the
%survey [33].

The usefulness of iterated integrals in data analysis has recently been realized,
see for example \cite{bib:LLN2012,bib:Gra2013,bib:KSHGL2017,bib:YLNSJC2017} and the introduction in \cite{bib:OxSigIntro}.
%
%\TODO{ Signatures are... [ 2 sentences] }
%
Let us demonstrate the appearance of iterated integrals on a very simple example.
Let $X: [0,T] \to \R^2$ be a smooth curve.
Say, we are looking for a feature describing this curve that is unchanged if one is handed a rotated version
of $X$. Maybe the simplest one that one can come up with is the (squared) total displacement length $|X_T - X_0|^2$.
Now,
\begin{align*}
  |X_T - X_0|^2
  &=
  (X^1_T - X^1_0)^2
  + 
  (X^2_T - X^2_0)^2 \\
  &=
  2 \int_0^T \left( X^1_r - X^1_0 \right) \dot X^1_r dr
  +
  2 \int_0^T \left( X^2_r - X^2_0 \right) \dot X^2_r dr \\
  &=
  2 \int_0^T \left( \int_0^r \dot X^1_u du \right) \dot X^1_r dr
  +
  2 \int_0^T \left( \int_0^r \dot X^2_u du \right) \dot X^2_r dr \\
  &=
  2 \int_0^T \int_0^r d X^1_u d X^1_r
  +
  2 \int_0^T \int_0^r d X^2_u d X^2_r,
\end{align*}
where we applied the fundamental theorem of calculus twice
and introduced the notation $dX^i_r = \dot{X}^i_r dr$.
We see that we have expressed this simple invariant in terms of iterated integrals of $X$;
the collection of which is usually called its \emph{signature}.
The aim of this work can be summarized as describing \emph{all} invariants that can be obtained in this way.
It turns out, when formulated in the right way, this search for invariants
reduces to classical problems in invariant theory.
We note that already in the early work of Chen (see for example \cite[Chapter 3]{bib:Che1957})
the topic of invariants arose, although a systematic study was missing (see also \cite{bib:Joh1962}).
%
%In some sense this work can hence be seen as the application of classical results
%to the setting of invariant feature selection for time series.

%%
%%The (geometric) interpretation of terms in the signature is a wide open question,
%%and we make
%We also produce mathematically new results.
%In particular, we state new (geometric) interpretations for certain terms of the signature (Section~\ref{sec:dXw1}).

The aim of this work is threefold.
Firstly, we adapt classical results in invariant theory
regarding non-commuting polynomials (or, equivalently, multilinear maps),
to our situation.
These results are spread out in the literature and sometimes need a little massaging.
%
%Firstly, it collects classical results in invariant theory
%regarding non-commuting polynomials (or, equivalently, multilinear maps),
%since they are spread out in the literature
%and sometimes need a little massaging.
%
Secondly, it lays out the usefulness of the signature of iterated integrals
in the search for invariants of $\ds$-dimensional curves.
We show, see Section~\ref{sec:discussion}, that certain ``integral invariants''
found in the literature are in fact of this type and we simplify their enumeration.
Lastly, we present new geometric insights into some entries found in the signature,
Section~\ref{sec:dXw1}.%
\footnote{The signature is notorious for being hard to interpret in geometric terms.}

The paper is structured as follows.
In the next section we introduce the signature of iterated integrals of a multidimensional curve,
as well as some algebraic language to work with it.
Based on this signature,
we present in Section~\ref{sec:gl} and Section~\ref{sec:so}
invariants to the general linear group and the special orthogonal group.
Both are based on classical results in invariants theory.
%In the case of the general linear group we are able to present a linear basis of the invariants.
%
For completeness, we present in Section~\ref{sec:permuations} the invariants to permutations,
which have been constructed in \cite{bib:BRRZ2005}.
In Section~\ref{sec:time} we show how to use all these invariants if
an additional (time) coordinate is introduced.
In Section~\ref{sec:discussion} we relate our work to the integral invariants of \cite{bib:FKK2010}
and demonstrate that the invariants presented there cannot be complete.
We formulate the conjecture of completeness for our invariants and point out open algebraic questions.

For readers who want to use these invariants without having to go into the technical results, we propose the following route.
The required notation is presented in the next section.
The invariants are presented in Proposition~\ref{prop:glInvariants}, Proposition~\ref{prop:soInvariants} and Proposition~\ref{prop:permInvariants}.
Examples are given in Section~\ref{sec:explicit} (in particular Remark~\ref{rem:interpretation}), Example~\ref{ex:soInvariants} and Example~\ref{ex:permInvariants}.
All these invariants are also implemented in the software package \cite{bib:Die2018}.
For a python package for calculating the iterated-integrals signature we propose using the package \verb|iisignature|, as described in \cite{bib:Rei2017}.

\section{The signature of iterated integrals}
\label{sec:signature}

%We are interested in invariants of multidimensional time series, that is,
%of a (discrete) stream of data points in some $\R^\ds$.
%Such a time series can be easily transformed into a (piecewise) smooth curve
%by linear interpolation. \todonotes{jd: put graphic}
%
%For this reason we will from now on assume given a
%multidimensional \textbf{curve} $X: [0,T] \to \R^\ds$ of bounded variation.
%

By a multidimensional \textbf{curve} $X$ we will denote a continuous mapping $X: [0,T] \to \R^\ds$ of bounded variation%
\footnote{
  The reader might prefer to just think of a (piecewise) smooth curve.
  %The standard example in applications will be the piecewise linear interpolation
  %of $\ds$-dimensional data measured at discrete time points.
}.
The aim of this work is to find features (i.e. complex or real numbers) describing such a curve that are invariant under
the general linear group, the group of rotations and the group of permutations.
Note that in practical situations one is usually presented with a discrete sequence of data points in $\R^\ds$, a multidimensional \textbf{time series}.
Such a time series can be easily transformed into a (piecewise) smooth curve
by linear interpolation.

It was proven in \cite{bib:Che1957} (see \cite{bib:HL2010} for a recent generalization)
that a curve $X$ is almost completely characterized by the collection of its iterated integrals%
\footnote
{
  Since $X$ is of bounded variation
  the integrals are well-defined using classical Riemann-Stieltjes integration (see for example Chapter 6 in \cite{bib:Rud1964}).
  This can be pushed much further though.
  In fact the following considerations are purely algebraic
  and hence hold for any curve for which a sensible integration theory (in particular: obeying integration by parts) exists.
  A relevant example is Brownian motion which, although being almost surely
  nowhere differentiable, nonetheless admits a stochastic (Stratonovich) integral.
  %see Section~\ref{sec:robustness}.
}
\begin{align*}
  \int_0^T \int_0^{r_n} \dots \int_0^{r_2} dX^{i_1}_{r_1} \dots dX^{i_n}_{r_n}, \qquad n \ge 1,\quad i_1, \dots i_n \in \{1,\dots,\ds\}.
\end{align*}
The collection of all these integrals is called the \textbf{signature}%
\footnote{ Also called the ``rough path signature''.}
of $X$.
%In this work we shall only be concerned with invariants that can be written in terms of these iterated integrals.
%In particular we are only interested in linear combinations of terms of the sigature that are invariant.
In a first step,
we can hence reduce the goal
\begin{quote}
  \textit{Find functions $\Psi: \text{curves} \to \R$ that are invariant under the action of a group $G$.}
\end{quote}
to the goal
\begin{quote}
  \textit{Find functions $\Psi: \text{signature of curves} \to \R$ that are invariant under the action of a group $G$.}
\end{quote}
%This is justified by the results in \cite{bib:Che1957,bib:HL2010} cited above.
%
By the Shuffle identity (Lemma~\ref{lem:shuffleIdentity}),
any polynomial function on the signature can be re-written as a linear function on the signature.
Assuming that arbitrary functions are well-approximated by polynomial functions,
we are lead to the final simplification, which is the goal of this paper
\begin{quote}
  \textit{Find \emph{linear} functions $\Psi: \text{signature of curves} \to \R$ that are invariant under the action of a group $G$.}
\end{quote}

%  \structure{Goal}\\
%  Find functions \fbox{$\phi: \text{ curves } \to \R$} that are invariant under $G \subset \GL$.
%
%  ~\\
%  ~\\
%  \pause
%
%  \structure{First simplification}\\
%  Find functions
%  \fbox{$\phi: \text{ signature of curves } \to \R$} that are invariant under $G$. \\
%  \emph{Justification}: Signature operator is (almost) one-to-one. %Boedihardjo/Geng/Lyons/Yang '14.
%  %%
%  % SAY justified by the result of Chen
%  %%
%
%  ~\\
%  ~\\
%  \pause
%
%  \structure{Second simplification}\\
%  Find \structure{linear} functions \fbox{$\phi: \text{ signature of curves } \to \R$} that are invariant under $G$. \\
%  %%
%  % SAY something on shuffle product
%  %%
%  \emph{Justification}: Shuffle identity.
%  ~\\
%  ~\\
%  \pause
%  $\Rightarrow$ \underline{Find elements in $\TS$ that are invariant under $G$.}

\subsection{Algebraic underpinning}

Let us introduce some algebraic notation in order to work with the collection of iterated integrals.
Denote by $\TC$ the space of formal power series in $\ds$ \textit{non-commuting} variables $x_1, x_2, \dots, x_{\ds}$.
We can conveniently store all the iterated integrals of the curve $X$ in $\TC$,
by defining the \textbf{signature} of $X$ to be
\begin{align*}
  S(X)_{0,T} := \sum x_{i_1} \dots x_{i_n} \int_0^T \int^{r_n} \dots \int_0^{r_2} dX^{i_1}_{r_1} \dots dX^{i_n}_{r_n}.
\end{align*}
Here the sum is taken over all $n \ge 0$ and all $i_1, \dots, i_n \in \{1,2,..,\ds\}$.
For $n=0$ the summand is, for algebraic reasons, taken to be the constant $1$.
%Note that $S(X)_{0,T}$ is an element of $\TC$.

The algebraic dual of $\TC$ is $\TS$, the space of polynomials%
\footnote{
  In contrast to a power series, a polynomial only has finitely many terms.
}
in $x_1, x_2, \dots, x_{\ds}$.
The dual pairing, denoted by $\langle \cdot, \cdot \rangle$ is defined by declaring all monomials to be orthonormal, 
so for example
\begin{align*}
  \Big\langle x_1  + 15\cdot x_1 x_2 - 2\cdot x_1 x_2 x_1, x_1 x_2 \Big\rangle = 15.
\end{align*}
Here, we write the element of $\TC$ on the left and the element of $\TS$ on the right.
We can ``pick out'' iterated integrals from the signature as follows
\begin{align*}
  \Big\langle S(X)_{0,T}, x_{i_1} \dots x_{i_n} \Big\rangle
  =
  \int_0^T \int^{r_n} \dots \int_0^{r_2} dX^{i_1}_{r_1} \dots dX^{i_n}_{r_n}.
\end{align*}

The space $\TC$ becomes an algebra by extending the usual product of monomials, denoted $\cdot$, to the whole space by bilinearity.
Note that $\cdot$ is non-commutative.

On $\TS$ we usually use the shuffle product $\shuffle$ which, on monomials, interleaves them in all order-preserving ways,
so for example
\begin{align*}
 x_1 \shuffle x_2 x_3
 =
 x_1 x_2 x_3
 +
 x_2 x_1 x_3
 +
 x_2 x_3 x_1.
\end{align*}
Note that $\shuffle$ is commutative.

Monomials, and hence homogeneous polynomials, have the usual concept of \textbf{order} or homogeneity.
For $n \ge 0$ we denote the projection on polynomials of order $n$ by $\pi_n$, so for example
\begin{align*}
  \pi_2 
  \left( x_1  + 15\cdot x_1 x_2 - 2\cdot x_1 x_2 x_1 \right)
  =
  15\cdot x_1 x_2.
\end{align*}
%A polynomial in $\pi_n \TS$ or $\pi_n \TC$ is said to have \textbf{order} $n$.
%
See \cite{bib:Reu1993} for more background on these spaces.

As mentioned above, every polynomial expression in terms of the signature
can be re-written as a linear expression in (different) terms of the signature.
This is the content of the following lemma, which is proven in \cite{bib:Ree1958} (see also \cite[Corollary 3.5]{bib:Reu1993}).
\begin{lemma}[Shuffle identity]
  \label{lem:shuffleIdentity}
  Let $X: [0,T] \to \R^d$ be a continuous curve of bounded variation, then for every $a,b \in T(\R^d)$
  \begin{align*}
    \Big\langle S(X)_{0,T}, a \Big\rangle
    \Big\langle S(X)_{0,T}, b \Big\rangle
    =
    \Big\langle S(X)_{0,T}, a \shuffle b \Big\rangle
  \end{align*}
\end{lemma}

\begin{remark}
  We have used this fact already in the introduction, where we confirmed by hand that
  \begin{align*}
    \Big( \Big\langle S(X)_{0,T}, x_1 \Big\rangle \Big)^2
    +
    \Big( \Big\langle S(X)_{0,T}, x_2 \Big\rangle \Big)^2
    &=
    2\ \Big\langle S(X)_{0,T}, x_1 x_1 \Big\rangle
    +
    2\ \Big\langle S(X)_{0,T}, x_2 x_2 \Big\rangle \\
    \Big(
      &  =
    \Big\langle S(X)_{0,T}, x_1 \shuffle x_1 \Big\rangle
    +
    \Big\langle S(X)_{0,T}, x_2 \shuffle x_2 \Big\rangle
    \Big).
  \end{align*}
\end{remark}

The concatenation of curves is compatible with the product on $\TC$ in the following sense
(for a proof, see for example \cite[Theorem 7.11]{bib:FV2010}).
\begin{lemma}[Chen's relation]
  \label{lem:chen}
  For curves $X: [0,T] \to \R^\ds, Y: [0,T] \to \R^\ds$ denote
  their concatenation
  \begin{align*}
    X \sqcup Y: [0,2T] \to \R^\ds,
  \end{align*}
  as $X_\cdot$ on $[0,T]$ and $Y_{\cdot-T} - Y_0 + X_T$ on $[T,2T]$.
  Then
  \begin{align*}
    S(X \sqcup Y)_{0,2T} = S(X)_{0,T} \cdot S(Y)_{0,T}.
  \end{align*}
\end{lemma}

%The fact that a curve is almost determined by it signature%
%\footnote{
%A fact that is similar to the well-known result
%that the collection of all moments of a compactly supported density
%completely determine that density (modulo modfications on sets of zero measure).
%}
%gives us the justification to restrict our search for features of a curve to functions of its signature.
%Since iterated integrals only depend on the increments of a curve, this already gives us translation invariance for free.
%
%We can reduce the complexity of our task even more: %as we shall see in Section~\ref{sec:explicit},
%by the shuffle identity (Lemma~\ref{lem:shuffleIdentity}),
%every polynomial in elements of the signature is in fact a linear function in (different) elements of the signature.
%%
%Hence, we will restrict our search to \textit{linear} functions of the signature that are invariant under a change of coordinates.

We will use the following fact repeatedly,
which also explains the commonly used name \textbf{tensor algebra} for $\TS$.
\newcommand{\bijection}{\mathsf{poly}}
\begin{lemma}
  \label{lem:oneToOne}
  The space of all multilinear maps on $\R^{\ds} \times \dots \times \R^{\ds}$ ($n$-times)
  is in a one-to-one correspondence with homogeneous polynomials of order $n$ in the non-commuting variables $x_1,\dots,x_\ds$
  by the following bijection
  \begin{align*}
    \psi \mapsto \bijection(\psi) := \sum_{i_1, \dots, i_n \in \{1,\dots,\ds\}} \psi(e_{i_1}, e_{i_2}, \dots, e_{i_n}) x_{i_1} \cdot x_{i_2} \cdot .. \cdot x_{i_n},
  \end{align*}
  with $e_i$ being the $i$-th canonical basis vector of $\R^\ds$.
\end{lemma}

\section{General linear group}
\label{sec:gl}

Let
\begin{align*}
  GL(\R^\ds) = \{ A \in \R^{\ds\times \ds} : \det( A ) \not= 0 \},
\end{align*}
be the general linear group of $\R^\ds$.

\begin{definition}
  \label{def:invariant}
  For $w \in \N$,
  we call $\phi \in \TS$ an \textbf{GL invariant of weight $w$} if
  \begin{align*}
    \Big\langle S(A X)_{0,T}, \phi \Big\rangle = (\det A)^w \Big\langle S(X)_{0,T}, \phi \Big\rangle
  \end{align*}
  for all $A \in \operatorname{GL}(\R^\ds)$.
\end{definition}

%\begin{remark}
%  Elements in $T(\R^2)$ that are invariant only to \emph{rotations} (in which case one has in particular $\det A = 1$)
%  were studied in \cite{bib:Die2013}.
%  In the case $\ds=2$ every invariant in the above sense is automatically rotation invariant.
%  The other immplication does not hold:
%  for example the (square of the) Euclidean norm%
%  \footnote{
%    Which, by the shuffle idenity (Lemma~\ref{lem:shuffleIdentity}), can indeed be found in the signature:
%    \begin{align*}
%      \left( X^1_T - X^1_0 \right)^2
%      +
%      \left( X^2_T - X^2_0 \right)^2
%      =
%      2\ \langle S(X)_{0,T}, x_1 x_1 + x_2 x_2 \rangle.
%    \end{align*}
%  }
%  of the increment of $X$
%  is rotation invariant but it is not invariant with respect to all linear transformations in the above sense.
%
%  %there are invariants to rotation that are not invariant in the sense of Definition~\ref{def:invariant}.
%  %
%  %This manifests itself for example in Section~\ref{sec:explicit}
%  %where we find $2$ invariants in the sense of Definition~\ref{def:invariant}
%  %for $\ds=2$ on level $4$
%  %whereas in \cite{bib:Die2013} it is shown that there exist $3$ rotation
%  %invariants on this level.
%\end{remark}

\begin{definition}
  \label{def:action}
  Define a linear action of $GL(\R^\ds)$ on $\TC$ and $\TS$,
  by specifying on monomials
  \begin{align*}
    A x_{i_1} .. x_{i_n}
    &:=
    \sum_j (A e_{i_1})_{j_1} x_{j_1} .. (A e_{i_n})_{j_n} x_{j_n} \\
    &=
    \sum_j A_{j_1 i_1} .. A_{j_n i_n} x_{j_1} .. x_{j_n}.
  \end{align*}
\end{definition}

\begin{lemma}
  For all $A \in \R^{\ds \times \ds}$ and any curve $X$,
  \begin{align*}
    \Big\langle S(A X)_{0,T}, \phi \Big\rangle
    =
    \Big\langle S(X)_{0,T}, A^\top \phi \Big\rangle.
  \end{align*}
\end{lemma}
\begin{proof}
  %First
  %\begin{align*}
  %  S( A X )
  %  &=
  %  \sum_n \sum_{i} x_{i_1} \dots x_{i_n} \int d(AX)^{i_1} \dots d(AX)^{i_n} \\
  %  &=
  %  \sum_n \sum_{i} \sum_{j} x_{i_1} \dots x_{i_n} \int dA_{i_1, j_1} X^{j_1} \dots d A_{i_n, j_n} X^{j_n} \\
  %  &=
  %  \sum_n \sum_{j} \sum_{i} A_{i_1 j_1} x_{i_1} \dots A_{i_n j_n} x_{i_n} \int d X^{j_1} \dots d  X^{j_n} \\
  %  &=
  %  \sum_n \sum_{j} A x_{j_1} \dots A x_{j_n} \int d X^{j_1} \dots d  X^{j_n}.
  %\end{align*}
  It is enough to verify this on monomials $\phi = x_{\ell_1} .. x_{\ell_m}$.
  Then
  \begin{align*}
    \Big\langle S( A X ), \phi \Big\rangle
    &=
    \sum_j A_{\ell_1 j_1} \dots A_{\ell_m j_m} \int d X^{j_1} \dots d  X^{j_m} \\
    &=
    \Big\langle S(X), \sum_j A_{\ell_1 j_1} x_{j_1} .. A_{\ell_m j_m} x_{j_m} \Big\rangle \\
    &=
    \Big\langle S(X), A^\top \phi \Big\rangle.
  \end{align*}
\end{proof}

We can simplify the concept of GL invariants further, using the next lemma.
\begin{lemma}
  \label{lem:spans}
  For $n\ge 1$
  \begin{align}
    \label{eq:GspansT}
    \operatorname{span} \{ \pi_n S(X)_{0,T} : X \text{ curve } \} = \pi_n \TC.
  \end{align}
\end{lemma}
\begin{proof}
  It is clear by definition that the left hand side of \eqref{eq:GspansT} is included in $\pi_n \TC$.
  We show the other direction and use ideas of \cite[Proposition 4]{bib:CF2010}. 
  Let $x_{i_n} \cdot \ldots \cdot x_{i_1} \in \pi_n \TC$ be given.
  Let $X$ be the piecewise linear path that results from the concatenation of the vectors $t_1 e_{i_1}, t_2 e_{i_2}$ up to $t_n e_{i_n}$,
  where $e_i, i=1,..,\ds$ is the standard basis of $\R^\ds$.
  Its signature is given by (see for example \cite[Chapter 6]{bib:FV2010})
  \begin{align*}
    S(X)_{0,1} = \exp( {t_n x_{i_n}} ) \cdot \ldots \cdot \exp( t_1 x_{i_1} ) =: \phi(t_1, \dots, t_n),
  \end{align*}
  where the exponential function is defined by its power series.
  Then
  \begin{align*}
    \frac{d}{dt_n} \dots \frac{d}{dt_1} \phi(0,\dots,0) = x_{i_n} \cdot \ldots \cdot x_{i_1}.
  \end{align*}
  Combining this with the fact that left hand side of \eqref{eq:GspansT} is a closed set we get that
  \begin{align*}
    x_{i_n} \cdot \ldots \cdot x_{i_1}
    \in 
    \spann\{ \pi_n( S(X) _{0,1} ) : X \text{ curve } \}.
  \end{align*}

  These elements span $\pi_n \TC$, which finishes the proof.
\end{proof}

Hence, $\phi$ is a $GL$ invariant of weight $w$ in the sense of Definition~\ref{def:invariant} if and only if
for all $A \in GL(\R^\ds)$
\begin{align*}
  A^\top \phi  = (\det A)^w \phi.
\end{align*}

Since the action respects homogeneity, we immediately obtain
that projections of invariants are invariants (take $B = (\det A)^{-w} A^\top$ in the following lemma):
\begin{lemma}
  \label{lem:projectionIsInvariant}
  If $\phi \in \TS$ satisfies
  \begin{align*}
    B \phi = \phi,
  \end{align*}
  for some $B \in GL(\R^\ds)$ then
  \begin{align*}
    B \pi_n \phi = \pi_n \phi,
  \end{align*}
  for all $n \ge 1$.
\end{lemma}
\begin{proof}
  By definition,
  the action of $GL$ on $\TS$ commutes with $\pi_n$.
\end{proof}

%\TODO{ do we want to switch to $A$ ?? }
%Since $GL(\R^\ds)$ is closed under taking the transpose, and since $\det A^\top = \det A$,
%$\phi$ is a $GL$ invariant of weight $w$ if and only if
%for all $A \in GL(\R^\ds)$
%\begin{align*}
%  A \phi  = (\det A)^w \phi.
%\end{align*}

In order to apply classical results in invariant theory, we use the
bijection $\bijection$ between multilinear functions and non-commuting polynomials,
given in Lemma~\ref{lem:oneToOne}.
\begin{lemma}
  \label{lem:multi}
  For $\psi: (\R^d)^{\times n} \to \R$ multilinear
  and
  $A \in GL(\R^d)$,
  \begin{align*}
    \bijection[ \psi( A \cdot ) ] = A^\top \bijection[ \psi ].
  \end{align*}
\end{lemma}
\begin{proof}
  \begin{align*}
    \bijection[ \psi( A \cdot ) ]
    &=
    \sum_i \psi( A e_{i_1}, .. A e_{i_n} ) x_{i_1} .. x_{i_n} \\
    &=
    \sum_{i,j} A_{j_1 i_1} .. A_{j_n i_n} \psi( e_{j_1}, .. e_{j_n} ) x_{i_1} .. x_{i_n} \\
    &=
    \sum_{j} \psi( e_{j_1}, .. e_{j_n} ) A^\top x_{j_1} .. x_{j_n} \\
    &=
    A^\top B \psi.
  \end{align*}
\end{proof}

The simplest multilinear function
\begin{align*}
  \Psi: (\R^d)^{\times n} \to \R,
\end{align*}
satisfying $\Psi( A v_1, .., A v_n ) = \det( A ) \Psi(v_1, .., v_n)$
that one can maybe think of, is the determinant itself. That is, $n=\ds$ and
\begin{align*}
  \Psi(v_1,..,v_n) = \det[ v_1 v_2 .. v_n ],
\end{align*}
where $v_1 v_2 .. v_n$ is the $d \times d$ matrix with columns $v_i$.
Up to a scalar this is in fact the only one, and it turns out that invariants of higher weight are
built only using determinants as a building block.

To state the following classical result, we introduce
the notion of Young diagrams, which play an important role in the
representation theory of the symmetric group.

Let $\lambda = (\lambda_1, .., \lambda_r)$ be a partition of $n \in \N$,
which we assume ordered as $\lambda_1 \ge \lambda_2 \ge .. \ge \lambda_r$.
We associate to it a \textbf{Young diagram}, which is an arrangement of $n$ boxes
into left-justified rows. There are $r$ rows, with $\lambda_i$ boxes in the $i$-th row.
For example, the partition $(4,2,1)$ of $7$ gives the Young diagram
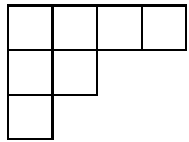
\begin{align*}
  \begin{ytableau}
    \  & \  & \  & \ \\
    \   & \  \\
    \ 
  \end{ytableau}
\end{align*}

A \textbf{Young tableau} is obtained by filling these boxes with the numbers $1, .., n$.
Continuing the example, the following is a Young tableau
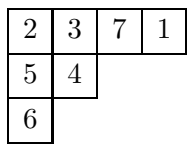
\begin{align*}
  \begin{ytableau}
    2  & 3 & 7  & 1 \\
    5  & 4  \\
    6 
  \end{ytableau}
\end{align*}

A Young tableau is \textbf{standard} if the values in every row are increasing (from left to right)
and are increasing in every column (from top to bottom). The previous tableau was not standard; the following is.
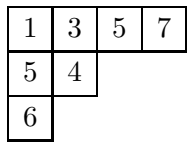
\begin{align*}
  \begin{ytableau}
    1  & 3 & 5  & 7 \\
    5  & 4  \\
    6 
  \end{ytableau}
\end{align*}

The following result is classical,
see for example Dieudonn\'e \cite[Section 2.5]{bib:DC1970}, \cite{bib:Wey1946} and \cite{bib:Gar1975},
none of which explicitly give a basis for the invariants though.

%There seem to be at least three ways to prove it.
%Weyl \cite{bib:Wey1946} uses the ``Capelli identities''
%and Dieudonn\'e \cite[Section 2.5]{bib:DC1970}
%uses the Young theory of the symmetric group.
%See \cite{bib:Gar1975} for a simple modern proof.
\begin{theorem}
  \label{thm:linearBasis}

   The space of multilinear maps
   \begin{align*}
     \psi: \underbrace{\R^\ds \times \dots \times \R^\ds}_{n \text{ times}} \to \R
   \end{align*}
   that satisfy
   \begin{align*}
     \psi(A v_1, A v_2, \dots, A v_n) = (\det A)^w \psi(v_1, v_2, \dots, v_n)
   \end{align*}
   for all $A \in \operatorname{GL}(V)$ and $v_1, \dots, v_n \in V$
   is non-empty if and only if $n = w d$ for some integer $w \ge 1$.

  In that case, a linear basis is given by
  \begin{align*}
    \{ v \mapsto \det[ v_{C_1} ]  .. \det[ v_{C_w} ] \}
  \end{align*}
  where
  $C_i$ are the columns of $\Sigma$,
  and
  $\Sigma$ ranges over all standard Young tableaux corresponding to the partition $\lambda = \underbrace{(w, w, .., w)}_{d \text{ times}}$ of $n$.

  Here, for a sequence $C = (c_1,..,c_\ds)$, $v_C$ denotes the matrix
  of column vectors $v_{c_i}$, i.e.
  \begin{align*}
    v_C = (v_{c_1}, .., v_{c_\ds}).
  \end{align*}
\end{theorem}
\begin{remark}
  A consequence of this theorem,
  is the existence of identities between products of determinants.
  %are identities of the following type:
  % 
  For example, for vectors $v_1, .., v_4 \in \R^2$, one can check by hand
  \begin{align*}
    \det[ v_1 v_4 ] \det[ v_2 v_3]
    =
    \det[ v_1 v_3 ] \det[ v_2 v_4]
    -
    \det[ v_1 v_2 ] \det[ v_3 v_4].
  \end{align*}
  This is why the product on the left-hand side here is not part of the basis in the previous lemma for $d=2, w=2$
  (compare Section \ref{sec:explicit}). 

  Identities of this type
  are called \emph{Pl\"ucker identities}.
  They have a long history and are a major ingredient in the representation theory of the symmetric group.
  %where they are also called syzegies.
  %
  The procedure of reducing certain products of determinants to a basic set of such products
  is called the \emph{straightening algorithm} \cite[Section 2.6]{bib:Sag2013}.
  See also 
  \cite{bib:Lec1993} and
  \cite{bib:SW1989}.
\end{remark}

\begin{remark}
  The only invariant for $d=2, w=1$ is
  \begin{align*}
    x_1 x_2 - x_2 x_1 = [x_1,x_2],
  \end{align*}
  a Lie polynomial.
  One can generally ask for invariant Lie polynomials \cite[Section 8.6.2]{bib:Reu1993}.
  This seems to be of no relevance to the application of invariant feature extraction for curves though.
\end{remark}

\begin{proof}
  Write $V = (\R^\ds)^*$, the dual space of $\R^\ds$.
  Every $\phi \in V^{\otimes n}$ that satisfies
  \begin{align}
    \label{eq:x}
    A \phi = (\det A)^w \phi,
  \end{align}
  clearly spans a one-dimensional irreducible representation of $GL(V)$.
  Hence, we need to investigate all one-dimensional irreducible representation 
  of $GL(V)$ contained in $V^{\otimes n}$ (and it will turn out that all of them satisfy \eqref{eq:x}).

  The (diagonal) action of $GL(V)$ on $V^{\otimes n}$ is best understood,
  by simultaneously studying the left action of $S_n$ on $V^{\otimes n}$ given
  by
  \begin{align*}
    \tau \cdot v_1 \otimes .. \otimes v_n
    :=
    v_{\tau^{-1}(1)} \otimes .. \otimes v_{\tau^{-1}(n)}.
  \end{align*}

  %If .. then .. spans a one-dimensional (and hence irreducible)
  %representation for $GL$.
  %It hence suffices to find all of them (and afterwards check
  %that they are of the form needed, i.e. that the determinant pops out).

  By Schur-Weyl duality, \cite[Theorem 6.4.5.2]{bib:Lan2012}, as $S_n \times GL(V)$ modules,
  \begin{align}
    \label{eq:schurWeyl}
    V^{\otimes n} \simeq \bigoplus_{\lambda \vdash n} S^\lambda \otimes V^\lambda,
  \end{align}
  where the sum is over integer partitions $\lambda$ of $n$,
  the $S^\lambda$ are irreducible representations of $S_n$, to be detailed below
  and the $V^\lambda$ are irreducible representations of $GL(V)$.
  The exact form of the latter is irrelevant here,
  we only need to know that $V^\lambda$ is one-dimensional if
  and only if $\lambda = (w,..,w)$, $d$-times, for some integer $w \ge 1$,
  \cite[p.21]{bib:DC1970}.
  This gives the condition $n = w d$ in the statement. We assume this to hold from now on.

  %We hence only need to realize $\dim S^\lambda$ copies worth of $V^\lambda$ in $V^{\otimes n}$,
  We are hence left with understanding the unique copy of the ``Specht module'' $S^\lambda$ inside of $V^{\otimes n}$.
  We sketch its classical construction.
  %We first sketch how the irreducible representations for $S_n$ are constructed.
  Let us recall that a \textbf{tabloid} is an equivalence class of Young tableaux modulo
  permutations leaving the set of entries in each row invariant \cite[Chapter 2]{bib:Sag2013}.%
  \footnote{
    One can also think of a tabloid as the following element of the vector space spanned by Young tableaux,
    \begin{align*}
      \{ t \} = \sum_{\pi} \pi t.
    \end{align*}
    Here the sum is over all permutations $\pi$ that leave the elements of each \emph{row} of $t$ unchanged.
  }
  For $t$ a Young tableau denote $\{ t \}$ its tabloid, so for example
  \begin{align*}
    \left\{
    \begin{ytableau}
      1  & 3 \\
      2  & 4
    \end{ytableau}
    \right\}
    =
    \left\{
    \begin{ytableau}
      1  & 3 \\
      4  & 2
    \end{ytableau}
    \right\}
    =
    \left\{
    \begin{ytableau}
      3  & 1 \\
      2  & 4
    \end{ytableau}
    \right\}
    =
    \left\{
    \begin{ytableau}
      3  & 1 \\
      4  & 2
    \end{ytableau}
    \right\}.
  \end{align*}

  The symmetric group $S_n$ acts on Young tableaux as
  \begin{align*}
    (\tau \cdot t)_{ij} := \tau( t_{ij} ).
  \end{align*}
  For example
  \begin{align*}
    (234) \cdot
    \begin{ytableau}
      2  & 4 \\
      1  & 3
    \end{ytableau}
    =
    \begin{ytableau}
      3  & 2 \\
      1  & 4
    \end{ytableau}
  \end{align*}
  It then acts on tabloids by $\tau \cdot \{t\} := \{ \tau \cdot t \}$.
  Define for a Young tableau $t$ 
  %\cite[Definition 2.3.2]{bib:Sag2013},
  \begin{align*}
    e_t := \sum_{\pi} \sign(\pi) \pi \cdot \{ t \},
  \end{align*}
  where the sum is over all $\pi \in S_n$ that leave the set of values in each \emph{column} invariant.
  For example with 
  \begin{align*}
    t =
    \begin{ytableau}
      1  & 2 \\
      3  & 4
    \end{ytableau}
  \end{align*}
  we get
  \begin{align*}
    e_t
    =   
    \{ t \}
    +
    (1 3) \cdot \{ t \}
    +
    (2 4) \cdot \{ t \}
    +
    (1 3) (2 4) \cdot \{t\}.
  \end{align*}

  Then
  \begin{align*}
    \operatorname{Irrep}_{(w,..,w)}
    := \operatorname{span}\{ e_t : t \text{ Young tableau of shape } (w,..,w) \}
  \end{align*}
  is an irreducible representation of $S_n$ and
  \begin{align*}
    \{ e_t : t \text{ standard Young tableau of shape $(w,..,w)$}  \},
  \end{align*}
  forms a basis \cite[Theorem 2.5.2]{bib:Sag2013}.
  This concludes the reminder on representation theory for $S_n$.

  %We now show that the invariants are an irreducible representation isomorphic to $\operatorname{Irrep}_{(w,..,w)}$.
  %%
  %First, $S_n$ acts on the left on $n$-multilinear functions by permutation of the arguments
  %%(note the inverse)
  %\begin{align*}
  %  \tau \cdot \Psi(v_1, .., v_n)
  %  :=
  %  \Psi(v_{\tau(1)}, .., v_{\tau(n)}).
  %  %\Psi(v_{\tau^{-1}(1)}, .., v_{\tau^{-1}(n)}).
  %\end{align*}
  %This means, for tensors $\Psi = w^*_1 \otimes .. \otimes w^*_n$ we have the left action%
  %\footnote{ Note the inverse; otherwise this would not be a left action.}
  %\begin{align*}
  %  \tau \cdot w^*_1 \otimes .. \otimes w^*_n
  %  =
  %  w^*_{\tau^{-1}(1)} \otimes .. \otimes w^*_{\tau^{-1}(n)}.
  %\end{align*}

  Define the map $\iota$ from the space of tabloids of shape $(w,..,w)$ into $V^{\otimes n}$ as follows,
  %For a tabloid of shape $(w,w,..,w)$ define
  \begin{align*}
    \iota( \{ t \} )
    :=
    e_{j_1}^* \otimes .. \otimes e_{j_n}^*,
  \end{align*}
  where $e_i^*$ is the canonical basis of $V$ and
  \begin{align*}
    j_\ell = i \qquad \Leftrightarrow \qquad \ell \in \text{ $i$-th row of } \{ t \}.
  \end{align*}
  For example
  \begin{align*}
    \iota\left(
      \left\{
        \begin{ytableau}
        1 & 2 & 5\\
        3 & 4 & 6
        \end{ytableau}
      \right\}
    \right)
    =
    e_1^* \otimes e_1^* \otimes e_2^* \otimes e_2^* \otimes e_1^* \otimes e_2^*.
  \end{align*}
  This is a homomorphism of $S_n$ representations.
  Indeed,
  \begin{align*}
    \iota( \tau \cdot \{t\} )
    =
    e_{j_1}^* \otimes .. \otimes e_{j_n}^*,
  \end{align*}
  with
  \begin{align*}
    j_\ell = i &\Leftrightarrow \ell \in \text{ $i$-th row of } \tau \cdot \{t\}.% \\
               %&\Leftrightarrow \tau^{-1}(\ell) \in \text{ $i$-th row of } \{t\}.
  \end{align*}
  %\begin{align*}
  %  j_\ell = i &\Leftrightarrow \ell \in \tau \cdot A_i \\
  %             &\Leftrightarrow \tau^{-1}( \ell ) \in A_i
  %\end{align*}
  On the other hand
  \begin{align*}
    \tau \cdot \iota( \{t\} )
    &=
    \tau \cdot e_{r_1}^* \otimes .. \otimes e_{r_n}^* \\
    &=
    %e_{r_{\tau^{-1}(1)}}^* \otimes .. \otimes e_{r_{\tau^{-1}(n)}}^*,
    e_{p_1}^* \otimes .. \otimes e_{p_n}^*,
  \end{align*}
  with $p_\ell := r_{\tau^{-1}(\ell)}$
  and
  \begin{align*}
    p_\ell = i &\Leftrightarrow r_{\tau^{-1}(\ell)} = i \\
               &\Leftrightarrow \tau^{-1}(\ell) \in \text{ $i$-th row of } \{t\} \\
               &\Leftrightarrow \ell \in \text{ $i$-th row of } \tau \cdot \{t\}.
               %&\Leftrightarrow \ell \in \tau \cdot A_i. 
    %r_\ell = i &\Leftrightarrow \ell \in A_i \\
    %           &\Leftrightarrow \tau(\ell) \in \tau \cdot A_i \\
    %           &\Leftrightarrow p_\ell = i.
  \end{align*}
  So indeed $\iota( \tau \cdot \{t\} ) = \tau \cdot \iota \{t\} )$,
  and $\iota$ is a homomorphism of $S_n$ representations.
  It is a bijection from the space of $(w,..,w)$ tabloids
  into the space spanned by the vectors
  \begin{align*}
    e_{i_1} \otimes .. \otimes e_{i_n} : \#\{ \ell : i_\ell = j \} = w, \quad j = 1, .., \ds.
  \end{align*}
  Restricting to $\operatorname{Irrep}_{(w,..,w)}$ then yields an isomorphism of irreducible $S_n$ representations.
  Hence, $\iota( \operatorname{Irrep}_{(w,..,w)} )$ is the (unique) realization of $S^\lambda$ inside of $V^{\otimes n}$ in \eqref{eq:schurWeyl}.
  We finish by describing its image.

  Consider the standard Young tableau $t_{first}$ of shape $(w,w,..,w)$ obtained by filling the columns from left to right, i.e.
  \begin{align*}
    t_{first}
    := 
    \raisebox{40pt}[80pt][80pt]{
    \ytableausetup{mathmode, boxsize=3em}
    \begin{ytableau}
      1   & \ds + 1 & ..& .. & .. \\
      2   & \ds + 2 & ..& .. & ..   \\
      ..  & ..      & ..& .. & .. \\
      \ds & 2 \ds   & ..& .. & n
    \end{ytableau}
    }
  \end{align*}

  Clearly, for any (standard) Young tableau $t$ there exists unique $\sigma_t \in S_n$ such that
  \begin{align*}
    \sigma_t \cdot t_{first} = t.
  \end{align*}

  We claim
  \begin{align*}
    \iota( e_t )
    =
    \Bigl( v \mapsto \det[ v_{\sigma_t(1)} .. v_{\sigma_t(\ds)} ] \cdot ... \cdot \det[ v_{\sigma_t((w-1)\ds + 1)} .. v_{\sigma_t(n)}] \Bigr).
  \end{align*}
  Indeed, since $\iota$ is a homomorphism of $S_n$ representation,
  \begin{align*}
    \iota( \sigma_t \cdot e_{t_{first}} )(v_1,..,v_n)
    &=
    \iota( e_{t_{first}} )(v_{\sigma_t(1)},..,v_{\sigma_t(n)})
  \end{align*}
  It remains to check
  \begin{align*}
    \iota( e_{t_{first}} ) = \det[v_1..v_\ds] ..\det[v_{(w-1)\ds+1} .. v_n].
  \end{align*}
  Every $\pi \in S^n$ that is column-preserving for $t_{first}$
  can be written as
  the product $\pi_1 \cdot .. \cdot \pi_w$,
  with $\pi_j$ ranging over the permutations of the entries of the $j$-th column $t_{first}$.
  %Note that $\ell$ is in the $i$-th row of $\pi \{ t_{first} \}$
  %iff 
  Then
  \newcommand{\modd}{\operatorname{mod}}
  \begin{align*}
    \iota( e_{t_{first}} )(v_1,..,v_n)
    &=
    \sum_\pi \sign \pi\ \iota( \pi \{t\} )( v_1, .., v_n ) \\
    &=
    \sum_{\pi_j} \prod_j \sign \pi_j\ \iota( \pi_1 .. \pi_w \{t\} )( v_1, .., v_n ) \\
    &=
    \sum_{\pi_j} \prod_j \sign \pi_j\ 
    e^*_{\pi^{-1}_1(1)} \otimes .. \otimes e^*_{\pi^{-1}_1(d)}
    \otimes
    e^*_{(\pi^{-1}_2(d+1) \modd d) + 1} \otimes\cdots
    %.. \otimes e^*_{(\pi^{-1}_2(2d) \modd d) + } \otimes
    \\&\qquad\qquad\cdots \otimes e^*_{(\pi^{-1}_w(n) \modd d) + 1} (v_1,..,v_n) \\
    &=
    \det[v_1..v_\ds] ..\det[v_{(w-1)d+1} .. v_n],
  \end{align*}
  as desired.

  %~\\
  %Hence $\iota( \operatorname{Irrep}_{(w,..,w)} ) = \invariants$ and
  %\begin{align*}
  %  \{ \iota(e_t) : t \text{ standard Young tableau of shape $(w,..,w)$ } \},
  %\end{align*}
  %forms a basis for $\invariants$, as claimed.
\end{proof}

Applying Lemma~\ref{lem:oneToOne} to Theorem~\ref{thm:linearBasis} we get the invariants in $\TS$.
\begin{proposition}
  \label{prop:glInvariants}
  A linear basis for the space of $GL$ invariants of order $n = w \ds$ is given by
  \begin{align*}
    \sum_{i_1,\dots,i_n \in \{1,\dots,d\}} g_\Sigma(i_1,i_2,\dots,i_n) x_{i_1} x_{i_2} \dots x_{i_n},
  \end{align*}
  where
  \begin{align*}
    g_\Sigma(v) = \det[ v_{C_1} ]  .. \det[ v_{C_w} ],
  \end{align*}
  where
  $C_i$ are the columns of $\Sigma$,
  and
  $\Sigma$ ranges over all standard Young tableaux corresponding to the partition $\lambda = \underbrace{(w, w, .., w)}_{d \text{ times}}$ of $n$.
\end{proposition}
\begin{remark}
  By Lemma~\ref{lem:projectionIsInvariant},
  for any invariant $\phi \in \TS$ and $n\ge 1$ we have that $\pi_n \phi$ is also invariant.
  Hence the previous theorem characterizes \emph{all} invariants we are interested in (Definition~\ref{def:invariant}),
  not just homogeneous ones.
\end{remark}

\begin{remark}
  \label{rem:hyperplane}
  Note that each of these invariants $\phi$ consists only of monomials that
  contain \emph{every} variable $x_1, \dots, x_\ds$ at least once.
  This implies that $\langle S(X)_{0,T}, \phi \rangle$
  consists only of iterated integrals that contain every component
  $X^1,\dots,X^d$ of the curve at least once.
  Hence, if at least one of these components is constant,
  the whole expression will be zero.

  Since $\phi$ is invariant, this implies that 
  $\langle S(X)_{0,T}, \phi \rangle = 0$ as soon as 
  there is some coordinate transformation under which one component is constant,
  that is whenever the curve $X$ stays in a hyperplane of dimension strictly less then $d$.

  One of the simplest curves in $\ds$ dimensions that does \emph{not} lie in any in a hyperplane of lower dimension is the \emph{moment curve}
  \begin{align*}
      t \mapsto (t,t^2,..,t^d).
  \end{align*}
  We will come back to this example in Lemma~\ref{lem:momentCurve}.

\end{remark}

\subsection{Examples}
\label{sec:explicit}

We will use the following short notation: 
\begin{align*}
  \word{i_1 \dots i_n} := x_{i_1} \cdot x_{i_2} \cdot .. \cdot x_{i_n}
\end{align*}
so, for example
\begin{align*}
  \word{1121} := x_1 x_1 x_2 x_1.
\end{align*}

We present the invariants described in Section~\ref{sec:signature}
for some special cases of $\ds$ and $w$.

%We give the linear basis for certain dimensions $\ds$ and orders $w$, given in Theorem~\ref{thm:linearBasis}.

\textbf{The case $\ds=2$}
%\textbf{The invariant of weight one, in dimension two}

Level $2$ ($w=1$)
\begin{align*}
  \word{12} - \word{21} 
\end{align*}

\begin{remark}
  \label{rem:interpretation}

  Let us make clear that from the perspective of data analysis,
  the ``invariant'' of interest is really the action of this element in $\TS$
  on the signature of a curve.

  In this example, the real number
  \begin{align*}
    \Big\langle S(X)_{0,T}, \word{12} - \word{21} \Big\rangle
    =  
    \int_0^T \int^{r_2} dX^1_{r_1} dX^2_{r_2}
    -
    \int_0^T \int^{r_2} dX^1_{r_1} dX^2_{r_2},
    %\int_0^T \int^{r_n} \dots \int_0^{r_2} dX^{i_1}_{r_1} \dots dX^{i_n}_{r_n}, \qquad n \ge 1,\quad i_1, \dots i_n \in \{1,\dots,\ds\}.
  \end{align*}
  changes only by the determinant of $A \in GL(\R^2)$ when
  calculating it for the transformed curve $A X$:
  \begin{align*}
    \Big\langle S(A X)_{0,T}, \word{12} - \word{21} \Big\rangle = \det(A)\ \Big\langle S(X)_{0,T}, \word{12} - \word{21} \Big\rangle.
  \end{align*}
\end{remark}

Level $4$ ($w=2$)
\begin{align*}
  & \word{1212} - \word{1221} - \word{2112} + \word{2121} \\
  & \word{1122} - \word{1221} - \word{2112} + \word{2211}
\end{align*}
\begin{remark}
  \label{rem:algebraicIndependence}
  This is a \textit{linear} basis of invariants in the fourth level.
  If one takes \textit{algebraic} dependencies into consideration,
  the set of invariants becomes smaller.
  To be specific,
  assume that one already has knowledge
  of the invariant of level $2$
  (i.e. $\langle S(X)_{0,T}, \word{12} - \word{21} \rangle$).
  If, say in a machine learning application,
  the learning algorithm can deal sufficiently well
  with nonlinearities, one should not be required to
  provide additionally the square of this number.
  In other words $|\langle S(X)_{0,T}, \word{12} - \word{21} \rangle|^2$ can also be assumed to be ``known''.
  But, by the shuffle identity (Lemma~\ref{lem:shuffleIdentity}),
  this can be written as
  \begin{align*}
    |\langle S(X)_{0,T}, \word{12} - \word{21} \rangle|^2&=\Big\langle S(X)_{0,T}, \word{12} - \word{21} \Big\rangle \cdot
    \Big\langle S(X)_{0,T}, \word{12} - \word{21} \Big\rangle
    \\&=
    \Big\langle S(X)_{0,T}, (\word{12} - \word{21}) \shuffle (\word{12} - \word{21}) \Big\rangle \\
    &=
    \Big\langle S(X)_{0,T}, 4\cdot \word{1122} - 4\cdot \word{1221} - 4\cdot \word{2112} + 4\cdot \word{2211} \Big\rangle.
  \end{align*}
  Now, given $\phi = 4\cdot \word{1122} - 4\cdot \word{1221} - 4\cdot \word{2112} + 4\cdot \word{2211}$ there is only one ``new'' independent invariant in the fourth level,
  namely $\word{1212} - \word{1221} - \word{2112} + \word{2121}$.

  A similar analysis can also be carried out for the following invariants, but we refrain from doing so, since
  it can be easily done with a computer algebra system.
\end{remark}

Level $6$ ($w=3$)
\begin{align*}
  & \word{121212} - \word{121221} - \word{122112} + \word{122121} - \word{211212} + \word{211221} + \word{212112} - \word{212121} \\
  & \word{112212} - \word{112221} - \word{122112} + \word{122121} - \word{211212} + \word{211221} + \word{221112} - \word{221121} \\
  & \word{121122} - \word{121221} - \word{122112} + \word{122211} - \word{211122} + \word{211221} + \word{212112} - \word{212211} \\
  & \word{112122} - \word{112221} - \word{122112} + \word{122211} - \word{211122} + \word{211221} + \word{221112} - \word{221211} \\
  & \word{111222} - \word{112221} - \word{121212} + \word{122211} - \word{211122} + \word{212121} + \word{221112} - \word{222111}
\end{align*}

\textbf{The case $\ds=3$}

Level $3$ ($w=1$)
\begin{align*}
  \word{123} - \word{132} - \word{213} + \word{231} + \word{312} - \word{321}
\end{align*}

Level $6$ ($w=2$)
\begin{scriptsize}
\begin{align*}
  & \word{123123} - \word{312132} + \word{312312} + \word{213132} - \word{213231} - \word{213123} + \word{321213} - \word{312321} - \word{132231} - \word{132123} \\
  &\quad - \word{321231} + \word{321132} + \word{132321} + \word{132213} + \word{231231} + \word{321321} + \word{213321} + \word{123231} + \word{231123} - \word{312213}\\
  &\quad - \word{321123} - \word{231132} + \word{213213} + \word{132132} + \word{312231} - \word{213312} - \word{231321} - \word{132312} - \word{123213} - \word{321312}\\
  &\quad + \word{312123} - \word{231213} + \word{231312} - \word{123321} + \word{123312} - \word{123132} \\
  & \word{211332} - \word{121332} + \word{212313} - \word{132231} - \word{313221} + \word{122331} - \word{211323} + \word{321132} + \word{132213} - \word{233121}\\
  &\quad + \word{323121} - \word{122313} + \word{213321} + \word{231123} - \word{312213} + \word{121323} - \word{321123} - \word{212331} - \word{231132} + \word{133221}\\
  &\quad - \word{131223} + \word{312231} + \word{233112} - \word{323112} - \word{311232} - \word{213312} + \word{313212} - \word{133212} + \word{131232} + \word{311223}\\
  &\quad - \word{232113} + \word{322113} - \word{123321} - \word{322131} + \word{123312} + \word{232131} \\
  & \word{112323} + \word{312132} + \word{332121} - \word{213132} + \word{211332} + \word{213123} - \word{321213} + \word{113232} - \word{331221} + \word{122331}\\
  &\quad - \word{211323} + \word{321231} + \word{223131} - \word{332112} - \word{132321} - \word{233121} + \word{331212} - \word{122313} - \word{123231} - \word{223113}\\
  &\quad + \word{133221} + \word{233112} - \word{311232} - \word{133212} - \word{221331} + \word{231321} + \word{132312} + \word{123213} - \word{312123} + \word{221313}\\
  &\quad + \word{311223} - \word{231312} + \word{322113} - \word{322131} - \word{113223} - \word{112332} \\
  & - \word{213132} + \word{211332} + \word{213231} - \word{133122} - \word{121332} - \word{321213} + \word{312321} + \word{212133} - \word{313221} + \word{122331}\\
  &\quad + \word{132123} + \word{131322} + \word{232311} - \word{322311} - \word{132321} - \word{311322} - \word{123231} + \word{323211} - \word{212331} + \word{133221}\\
  &\quad - \word{122133} - \word{131223} - \word{211233} + \word{233112} + \word{313122} - \word{233211} - \word{323112} + \word{321312} - \word{312123} + \word{231213}\\
  &\quad + \word{311223} + \word{121233} - \word{231312} - \word{232113} + \word{322113} + \word{123132} \\
  & \word{123123} + \word{312312} + \word{112233} + \word{211332} - \word{133122} - \word{132231} - \word{331221} + \word{122331} - \word{332112} - \word{322311}\\
  &\quad + \word{231231} + \word{321321} + \word{332211} - \word{311322} - \word{312213} + \word{331122} - \word{223113} - \word{321123} - \word{231132} + \word{213213}\\
  &\quad + \word{133221} + \word{132132} - \word{122133} - \word{211233} + \word{233112} - \word{233211} - \word{213312} - \word{221331} + \word{221133} + \word{311223}\\
  &\quad + \word{223311} + \word{322113} - \word{123321} - \word{113223} + \word{113322} - \word{112332}
\end{align*}
\end{scriptsize}
\textbf{The case $\ds=4$}

Level $4$ ($w=1$)
\begin{align*}
  &\word{1234} - \word{1243} - \word{1324} + \word{1342} + \word{1423} - \word{1432} - \word{2134} + \word{2143}\\&\qquad + \word{2314} - \word{2341} - \word{2413} + \word{2431} + \word{3124}  
   - \word{3142} - \word{3214} + \word{3241} \\&\qquad+ \word{3412} - \word{3421} - \word{4123} + \word{4132} + \word{4213} - \word{4231} - \word{4312} + \word{4321}
\end{align*}

%\subsubsection{\texorpdfstring{The case $\ds=2$, $w=1$}{The case d=2, w=1}}
%\subsection{\texorpdfstring{The case $\ds=2$, $w=1$}{The case d=2, w=1}}
%\subsection{The case $\ds=2$, $w=1$}
\subsection{The invariant of weight one, in dimension two}
\label{sec:d2w1}

\textbf{Geometric interpretation}\\
The invariant for $\ds=2,w=1$, namely $\phi = x_1 x_2 - x_2 x_1$
has a simple geometric interpretation: it picks out (two times)%
\footnote
{
  The prefactor $1/2$ is irrelevant, so we will speak of $\phi$ and also of $\tfrac12 \phi$
  as picking out the area.
}
the area (signed, and with multiplicity)
between the curve $X$ and the cord spanned between its starting and endpoint (compare Figure~\ref{fig:area}).
\begin{figure}[h]
  \centering
  \begin{tikzpicture}
\begin{scope}[decoration={
	markings,
	mark=at position 0.3 with {\arrow{Latex[length=4mm]};}}]
\filldraw [color=blue!11] plot [smooth, tension=0.8]
coordinates{(2,0)(0,2)(-3,2)(-7,-1)(-8,1)};
\draw[thick, postaction={decorate}] plot [smooth, tension=0.8] coordinates{(2,0)(0,2)(-3,2)(-7,-1)(-8,1)};
\draw [thick,dashed] (2,0) -- (-8,1);
\node at (-1.5,1.4) {\large $+$};
\node at (-7,0.1) {\large $-$};
\node at (-1,2.7) {\large $X$};
\end{scope}
\end{tikzpicture}
  \caption{
    A curve $X = (X^1, X^2)$ is shown,
    with shaded area given by
    $\frac{1}{2} \langle S(X)_{0,T}, x_1 x_2 - x_2 x_1 \rangle
    =
    \frac{1}{2} \int_0^T \int_0^{r_2} dX^1_{r_1} dX^2_{r_2} - 
    \frac{1}{2} \int_0^T \int_0^{r_2} dX^2_{r_1} dX^1_{r_2}$.
    }
    \label{fig:area}
\end{figure}
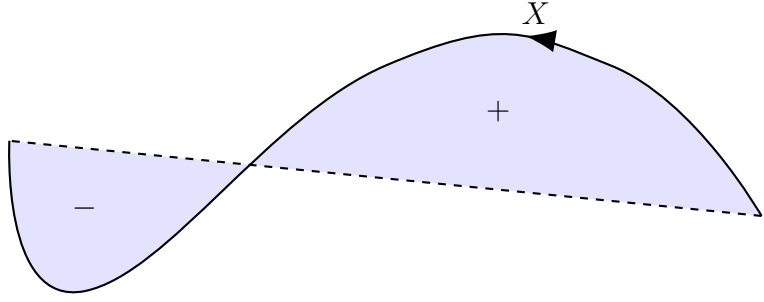
For (smooth) non-intersecting curves, this follows from Green's theorem \cite[Theorem 10.33]{bib:Rud1964}.
For self-intersecting curves, the mathematically most convenient definition
of ``signed area'' is the integral (in the plane) of its winding number.
The claimed relation to the invariant $\phi$ is for example proven in 
\cite[Proposition 1]{bib:LY2006}.

\textbf{Connection to correlation}\\
Assume that $X$ is a continuous curve,
piecewise linear between some time points $t_i,\ i=0,\dots,n$.\footnote
{
  The standard example is a curve that is discretely observed at times $t_i$
  and linearly interpolated in between.
}
The area is then explicitly calculated as
\begin{align*}
  &\int_0^T \int_0^r d X^1_u dX^2_r - \int_0^T \int_0^r dX^2_u dX^1_r \\
  &\qquad=
  \int_0^T \left( X^1_r - X^1_0 \right) dX^2_r - \int_0^T \left( X^2_r - X^2_0 \right) dX^1_r \\
  &\qquad=
  \frac12 \sum_{i=0}^{n-1} \left( X^1_{t_{i+1}} - X^1_{t_0} + X^1_{t_i} - X^1_{t_0} \right) \left(X^2_{t_{i+1}} - X^2_{t_i} \right)
  \\&\qquad\qquad-
  \frac12 \sum_{i=0}^{n-1} \left( X^2_{t_{i+1}} - X^2_{t_0} + X^2_{t_i} - X^2_{t_0} \right) \left(X^1_{t_{i+1}} - X^1_{t_i} \right) \\
  &\qquad=
  \sum_{i=0}^{n-1} X^1_{t_i} \left[ X^2_{t_{i+1}} - X^2_{t_0} \right] - \sum_{i=0}^{n-1} \left[ X^1_{t_{i+1}} - X^1_{t_0} \right] X^2_{t_i}
  %+ X^1_{t_0} \left[ X^2_{t_n} - X^2_{t_0} \right]
  %- X^2_{t_0} \left[ X^1_{t_n} - X^1_{t_0} \right]
  \\
  &\qquad=
  \operatorname{Corr}(X^2-X^2_{t_0},X^1)_1 - \operatorname{Corr}(X^1-X^1_{t_0},X^2)_1
  %+ X^1_{t_0} \left[ X^2_{t_n} - X^2_{t_0} \right]
  %- X^2_{t_0} \left[ X^1_{t_n} - X^1_{t_0} \right]
\end{align*}
Here, for two vectors $a,b$ of length $n$
\begin{align*}
  \operatorname{Corr}(a,b)_1 := \sum_{i=0}^{n-1} a_{i+1} b_i,
\end{align*}
the lag-one cross-correlation, which
is a commonly used feature in data analysis, see for example \cite[Chapter 13.2]{bib:PTVF2007}.
%
% related: \url{https://en.wikipedia.org/wiki/Shoelace_formula}
%
In particular, if the curve starts at $0$, we have
\begin{align*}
  \int_0^T \int_0^r d X^1_u dX^2_r - \int_0^T \int_0^r dX^2_u dX^1_r
  =
  \operatorname{Corr}(X^2,X^1)_1 - \operatorname{Corr}(X^1,X^2)_1,
\end{align*}
which is an antisymmetrized version of the lag-one cross-correlation.

\begin{remark}
  Note that it is immediate that the antisymmetrized version of the lag $\tau$ cross-correlation,
  $\tau \ge 2$ are also $GL(\R^2)$ invariants of the curve.
  Where they can be found in the signature $S(X)$ is unknown to us.
\end{remark}

%\subsubsection{\texorpdfstring{The case $\ds=3$, $w=1$}{The case d=3, w=1}}
%\subsection{\texorpdfstring{The case $\ds=3$, $w=1$}{The case d=3, w=1}}

%\subsection{A detailed look at some special cases}
%We look at the (geometric) interpretation of some of the $GL$ invariants.
%%JR: I don't want to "spend" a level of the hierarchy on this, because I want a level to be available in my vol stuff

%\subsubsection{\texorpdfstring{$\ds$ arbitrary, $w=1$}{d arbitrary, w=1}}
%\subsection{The case \texorpdfstring{$\ds$ arbitrary, $w=1$}{d arbitrary, w=1}}
%\subsection{The case $\ds$ arbitrary, $w=1$}
\subsection{The invariant of weight one, in any dimension}
\label{sec:dXw1}

Whatever the dimension $\ds$ of the curve's ambient space, the space of invariants of weight $1$ has dimension $1$ and is spanned by
\newcommand{\firstInvariant}[1]{\operatorname{Inv}_{#1}}
\begin{align}
  \label{eq:firstInvariant}
  \firstInvariant{\ds} :=
  \firstInvariant{\ds}(x_1,..,x_\ds) :=
  \sum_{\sigma \in S_\ds}
  \sign(\sigma)\
  x_{\sigma(1)} .. x_{\sigma(d)}
  =
  \det
  \begin{pmatrix}
    x_1 & .. & x_\ds \\
    .. & ..  & .. \\
    x_1 & .. & x_\ds
  \end{pmatrix}.
\end{align}
Here, for a matrix $C$ of non-commuting variables,
\begin{align*}
  \det C := \sum_\tau \sign \tau \prod_i C_{i \tau(i)}.
\end{align*}

This invariant is of homogeneity $\ds$.
The following lemma tells us that %in odd dimension $\ds$
we can write
$\firstInvariant{\ds}$ in terms of expressions on lower homogeneities.
%although formally on level $\ds$,
%can be written in terms of expressions on level $1$ and $2$.
%
\newcommand{\insertAfter}{\mathsf{InsertAfter}}
To state it, we first define the operation $\insertAfter(x_i,r)$ %\phi
on monomials %$\phi$ 
of order $n \ge r$, 
as the insertion of the variable $x_i$ after position $r$, and extend it linearly.
For example
\begin{align*}
  \insertAfter(x_1,1) \firstInvariant{2}(x_2,x_3)
  &=
  \insertAfter(x_1,1) \Big( x_2 x_3 - x_3 x_2 \Big) \\
  &=
  x_2 x_1 x_3 - x_3 x_1 x_2.
\end{align*}

\begin{lemma}
  \label{lem:dxw1}

  In any dimension $\ds$
  and for any $r=0,1,..,\ds-1$
  \begin{align*}
    \firstInvariant{\ds}(x_1,..,x_\ds)
    =
    (-1)^{r}
    \sum_{j=1}^{\ds}
    (-1)^{j+1} \insertAfter(x_j,r) \firstInvariant{\ds-1}( x_1, .., \widehat{x_j} .., x_\ds ),
  \end{align*}
  where $\widehat{x_j}$ denotes the omission of that argument.

  For $\ds$ odd,
  \begin{align*}
    \firstInvariant{\ds}(x_1,..,x_\ds)
    &=
    \sum_{j=1}^\ds
    (-1)^{j+1}
    x_j \shuffle \firstInvariant{\ds-1}( x_1, .., \widehat{x_j} .., x_\ds ).
  \end{align*}
\end{lemma}
\begin{remark}
  For completeness, we also note the related \emph{de Bruijn's formula}.
  For $\ds$ even,
  \begin{align*}
    \firstInvariant{\ds}(x_1,..,x_\ds)
    =
    \operatorname{Pf}_\shuffle[ A ],
  \end{align*}
  where
  \begin{align*}
    A_{ij}
    = 
    \firstInvariant{2}(x_i,x_j),
  \end{align*}
  and the Pfaffian (with respect to the shuffle product), is
  \begin{align*}
    \operatorname{Pf}_\shuffle[ A ]
    =
    \frac{1}{2^{\ds/2} (\ds/2)!}
    \sum_{\sigma \in S_{\ds}} \sign(\sigma)
    %\prod_{i=1}^{\ds/2}
    A_{\sigma(1),\sigma(2)}
    \shuffle
    A_{\sigma(3),\sigma(4)}
    \shuffle
    ..
    \shuffle
    A_{\sigma(\ds-1),\sigma(\ds)}.
  \end{align*}

  For a proof see \cite{bib:DB1955}
  and \cite{bib:LT2002}.
\end{remark}

\begin{proof}
  The first statement follows from expressing the determinant in~\ref{eq:firstInvariant}
  in terms of minors with respect to the row $r+1$
  (since the $x_i$ are non-commuting, this does not work with columns!).

  We demonstrate the proof for the second statement on the case $\ds=3$.
  Applying the first statement, we get
  \begin{align*}
    \firstInvariant{3}
    &=
    x_1 (x_2 x_3 - x_3 x_2) 
    -
    x_2 (x_1 x_3 - x_3 x_1) 
    +
    x_3 (x_1 x_2 - x_2 x_1) \\
    %%%%%%%%%%%%%
    &=
    -
    \left(
    x_2 x_1 x_3 - x_3 x_1 x_2
    -
    (x_1 x_2 x_3 - x_3 x_2 x_1)
    +
    x_1 x_3 x_2 - x_2 x_3 x_1 \right) \\
    %%%%%%%%%%%%%
    &=
    (x_2 x_3 - x_3 x_2) x_1
    -
    (x_1 x_3 - x_3 x_1) x_2 
    +
    (x_1 x_2 - x_2 x_1) x_3.
  \end{align*}

  Summing up and adding a $0$, we get
  \begin{align*}
    3 \firstInvariant{3}
    &=
    \left[
    x_1 (x_2 x_3 - x_3 x_2) 
    -
    x_2 (x_1 x_3 - x_3 x_1) 
    +
    x_3 (x_1 x_2 - x_2 x_1) \right] \\
    &\quad
    + 
    \left[
    x_2 x_1 x_3 - x_3 x_1 x_2
    -
    (x_1 x_2 x_3 - x_3 x_2 x_1)
    +
    x_1 x_3 x_2 - x_2 x_3 x_1 \right] \\
    &\quad
    +
    \left[
    (x_2 x_3 - x_3 x_2) x_1
    -
    (x_1 x_3 - x_3 x_1) x_2 
    +
    (x_1 x_2 - x_2 x_1) x_3 \right]\\
    &\quad-
    2 
    \left[
    x_2 x_1 x_3 - x_3 x_1 x_2
    -
    (x_1 x_2 x_3 - x_3 x_2 x_1)
    +
    x_1 x_3 x_2 - x_2 x_3 x_1 \right] \\
    &=
    x_1 \shuffle (x_2 x_3 - x_3 x_2)
    -
    x_2 \shuffle (x_1 x_3 - x_3 x_1)
    +
    x_3 \shuffle (x_1 x_2 - x_2 x_1)
    -
    2 \firstInvariant{3},
  \end{align*}
  and the result follows.
\end{proof}

An immediate consequence is the following lemma.
\begin{lemma}
  \label{lem:closedCurveOddDimension}
  If the ambient dimension $\ds$ is odd and the curve $X$ is closed (i.e. $X_T = X_0$) then
  \begin{align*}
    \Big\langle S(X)_{0,T}, \firstInvariant\ds \Big\rangle = 0.
  \end{align*}
\end{lemma}
\begin{proof}
  By Lemma~\ref{lem:dxw1} and then by the shuffle identity (Lemma~\ref{lem:shuffleIdentity})
  \begin{align*}
    \Big\langle S(X)_{0,T}, \firstInvariant\ds \Big\rangle
    &=
    \sum_{j=1}^d
    \Big\langle S(X)_{0,T}, 
    (-1)^{j+1}
    x_j \shuffle \firstInvariant{\ds-1}( x_1, .., \widehat{x_j} .., x_\ds )  \Big\rangle \\
    &=
    \sum_{j=1}^d
    (-1)^{j+1}
    \Big\langle S(X)_{0,T}, x_j \Big\rangle
    \Big\langle S(X)_{0,T}, \firstInvariant{\ds-1}( x_1, .., \widehat{x_j} .., x_\ds )  \Big\rangle \\
    &= 0,
  \end{align*}
  since the increment $\Big\langle S(X)_{0,T}, x_j \Big\rangle = X^j_T - X^j_0$ is zero for all $j$.
\end{proof}

In even dimension we have the phenomenon
that closing a curve does not change the value of the invariant.
\begin{lemma}
  \label{lem:lastPoint}
  If the ambient dimension $\ds$ is even, then for any curve $X$
  \begin{align*}
    \Big\langle S(X)_{0,T}, \firstInvariant\ds \Big\rangle
    =
    \Big\langle S(\bar X)_{0,T}, \firstInvariant\ds \Big\rangle,
  \end{align*}
  where $\bar X$ is $X$ concatenated with the straight line connecting $X_T$ to $X_0$.
\end{lemma}
\begin{proof}
  Let $\bar X$ be parametrized on $[0,2T]$ as follows: $\bar X = X$ on $[0,T]$
  and it is the linear path connecting $X_T$ to $X_0$ on $[T,2T]$.
  By translation invariance we can assume $X_0 = 0$ and by $GL(\R^\ds)$-invariance that $X_T$ lies on the $x_1$ axis.
  Then the only component of $\bar X$ that is non-constant on $[2T,T]$ is the first one, $\bar X^1$.

  By Lemma~\ref{lem:dxw1}
  \begin{align*}
    \firstInvariant{\ds}
    =
    -
    \sum_{j=1}^\ds
    (-1)^{j+1}
    \firstInvariant{\ds-1}(x_1, .., \hat x_j, .. x_d) x_j.
  \end{align*}
  Letting the summands act on $S(\bar X)_{0,t}$ we get $\pm 1$ times
  \begin{align*}
    \int_0^t \Big\langle S(\bar X)_{0,r} \firstInvariant{\ds-1}(x_1, .., \hat x_j, .. x_d) \Big\rangle d\bar X^j_r.
  \end{align*}
  For $j\not=1$ these expressions are constant on $[T,2T]$, since we arranged things so that those
  $\bar X^j$ do not move on $[T,2T]$.
  But also for $j=1$ this expression is constant on $[T,2T]$.
  Indeed, the integrand 
  \begin{align*}
    \Big\langle S(\bar X)_{0,r} \firstInvariant{\ds-1}(x_2, x_3, .., x_d) \Big\rangle,
  \end{align*}
  is zero on $[T,2T]$, since $X$, projected on the $x_2-..-x_d$ hyperplane, is a closed curve,
  and so Lemma~\ref{lem:closedCurveOddDimension} applies.
\end{proof}

\begin{lemma}
  \label{lem:invariantVolume}
  Let $X$ be the piecewise linear curve
  through $p_0,..,p_{\ds} \in \R^\ds$.
  Then
  \begin{align*}
    \Big\langle S(X)_{0,T}, \firstInvariant\ds \Big\rangle
    =
   \det\left[
     \begin{matrix}
       1 & 1 & .. & 1 \\
       p_0 & p_1 & .. & p_{\ds}
     \end{matrix}
   \right]
  \end{align*}
\end{lemma}
\begin{proof}
  %{\color{red} WHAT? Both sides are invariant to adding the same vector to all $p_i$.}
  First, for any $v \in \R^\ds$,
  \begin{align*}
   \det\left[
     \begin{matrix}
       1 & 1 & .. & 1 \\
       p_0 + v & p_1 + v & .. & p_{\ds} + v
     \end{matrix}
   \right]
   =
   \det\left[
     \begin{matrix}
       1 & 1 & .. & 1 \\
       p_0 & p_1 & .. & p_{\ds}
     \end{matrix}
   \right].
  \end{align*}
  %This can be seen by expanding the determinant in terms of minors with respect to the first row, and
  %then using multilinearity.
  %
  Since the signature is also invariant to translation, we can therefore assume $p_0 = 0$.
  Now both sides of the statement transform the same way under the action of $GL(\R^d)$ on
  the points $p_1,..p_\ds$.
  It is then enough to prove this for
  \begin{align*}
    p_0 &= 0 \\
    p_1 &= e_1 \\
    p_2 &= e_1 + e_2 \\
        &.. \\
    p_{\ds} &= e_1 + .. + e_\ds.
  \end{align*}

  Now, for this particular choice of points the right hand side is clearly equal to $1$.
  For the left hand side, the only non-zero term is
  \begin{align*}
    \Big\langle S(X)_{0,T}, \word{12} .. \word{d} \Big\rangle
    &=
    \int dX^1 .. dX^d \\
    &= 1.
  \end{align*}
\end{proof}

The modulus of the determinant 
\begin{align*}
   \det\left[
     \begin{matrix}
       1 & 1 & .. & 1 \\
       0 & p_1 & .. & p_{\ds}
     \end{matrix}
   \right]
   =
   \det\left[
     \begin{matrix}
       p_1 & .. & p_{\ds}
     \end{matrix}
   \right]
\end{align*}
gives the Lebesgue measure of the parallelepiped spanned by the vectors $p_1-p_0,..,p_\ds-p_0$.
The polytope spanned by the points $p_0,p_1,..,p_\ds$ fits $\ds!$ times into that parallelepiped.
We hence have the relation to classical volume as follows.
\begin{lemma}
  \label{lem:relationToClassicalVolume}
  Let $p_0,..,p_{\ds} \in \R^\ds$,
  then
  \begin{align*}
    | \operatorname{Convex-Hull}(p_0,..,p_\ds) |
    =
    \frac{1}{\ds!}
   \left|\det\left[
     \begin{matrix}
       1 & 1 & .. & 1 \\
       p_0 & p_1 & .. & p_{\ds}
     \end{matrix}
   \right]\right|
  \end{align*}
\end{lemma}

\begin{lemma}
  \label{lem:generalPiecewiseLinearCurve}
  Let $X$ be the piecewise linear curve through,
  $p_0,..,p_n \in \R^\ds$, with $n \ge \ds$.
  Then,
  \begin{align}
    \label{eq:invDet}
    \Big\langle S(X)_{0,T}, \firstInvariant\ds \Big\rangle
    =
    \sum_i
     \det\left[
       \begin{matrix}
         1 & 1 & .. & 1 \\
         p_{i_0} & p_{i_1} & .. & p_{i_{\ds}}
       \end{matrix}
     \right].
  \end{align}
  Here, for $\ds$ even, the subsequences $i$ are chosen as follows:
  \begin{align*}
    i_0 = 0
  \end{align*}
  and $i_1,..,i_{\ds}$ ranges over all possible increasing
  subsequences of $1,2,..,n$ such that
  for $\ell$ odd: $i_\ell + 1 = i_{\ell+1}$

  For $\ds$ odd, they are chosen as follows:
  \begin{align*}
    i_0 &= 0 \\
    i_{\ds} &= n,
  \end{align*}
  and $i_1,..,i_{\ds-1}$ ranges over all possible increasing
  subsequences of $1,2,..,n-1$ such that
  for $\ell$ odd: $i_\ell + 1 = i_{\ell+1}$
\end{lemma}
\begin{remark}
  In both the odd and the even case,
  there are
  \begin{align*}
    \binom{\lfloor d/2 \rfloor + n - \ds - 1 }{ n - d -1 }
  \end{align*}
  indices summed over.
\end{remark}
\begin{example}
  For $\ds=2$, $n=5$ we get the subsequences
  \begin{align*}
  [0, 1, 2] \\
  [0, 2, 3] \\
  [0, 3, 4]
  \end{align*}

  For $\ds=4$, $n=7$ we get the subsequences
  \begin{align*}
  [0, 1, 2, 3, 4] \\
  [0, 1, 2, 4, 5] \\
  [0, 1, 2, 5, 6] \\
  [0, 2, 3, 4, 5] \\
  [0, 2, 3, 5, 6] \\
  [0, 3, 4, 5, 6]
  \end{align*}

  For $\ds=5$, $n=8$ we get the subsequences
  \begin{align*}
  [0, 1, 2, 3, 4, 7] \\
  [0, 1, 2, 4, 5, 7] \\
  [0, 1, 2, 5, 6, 7] \\
  [0, 2, 3, 4, 5, 7] \\
  [0, 2, 3, 5, 6, 7] \\
  [0, 3, 4, 5, 6, 7]
  \end{align*}
\end{example}
\begin{proof}
  \textbf{The case $\ds=2$}\\
  %We are summing up the determinants corresponding to
  %the points $(p_0,p_1,p_2), (p_0,p_2,p_3), .., (p_0,p_{n-1}, p_n)$.
  Let $X$ be the curve through the points $p_0,p_1,..,p_n$,
  we can write it as concatenation
  of the curves $X^{(i)}$ where $X^{(i)}$ is the curve through
  the points $p_0 p_i, p_{i+1}, p_0$.
  The interval of definition for these curves (and all curves
  in this proof) do not matter, so we omit the subscript of $S(.)$.
  Then, by Chen's lemma (Lemma \ref{lem:chen})
  \begin{align*}
    \Big\langle S(X), \word{12} - \word{21} \Big\rangle
    &=
    \Big\langle S(X^{(n-1)}) \cdot .. \cdot S(X^{(1)}), \word{12} - \word{21} \Big\rangle \\
    &=
    \sum_{i=1}^{n-1} \Big\langle S(X^{(i)}), \word{12} - \word{21} \Big\rangle.
  \end{align*}
  For the last equality we used that
  \begin{align*}
    \Big\langle g h, \word{12} - \word{21}\Big\rangle
    =
    \Big\langle g, \word{12} - \word{21}\Big\rangle
    +
    \Big\langle h, \word{12} - \word{21}\Big\rangle
    +
    \Big\langle g, \word{1} \Big\rangle \Big\langle h, \word{2} \Big\rangle
    -
    \Big\langle g, \word{2} \Big\rangle \Big\langle h, \word{1} \Big\rangle,
  \end{align*}
  and that the increments of all curves $X^{(i)}$ are zero.
  Now by Lemma~\ref{lem:lastPoint}
  we can omit the last straight line in every $X^{(i)}$ and hence
  by Lemma~\ref{lem:invariantVolume}
  \begin{align*}
    \langle S(X^{(i)}), \word{12} - \word{21} \rangle
    =
     \det\left[
       \begin{matrix}
         1       & 1   & 1 \\
         p_{0} & p_{i} & p_{i+1}
       \end{matrix}
     \right],
  \end{align*}
  which finishes the proof for $\ds=2$.

  ~\\

  Now assume the statement is true for all dimensions strictly smaller than some $\ds$. We show it is true for $\ds$.

  \textbf{$\ds$ is odd }\\
  As before we can assume $p_0 = 0$ and that $p_n$ lies on the $x_1$ axis.
  Every sequence summed over on the right-hand side of \eqref{eq:invDet}
  is of the form $i = (0,...,n)$.
  For each of those, we calculate
  \begin{align*}
     \det\left[
       \begin{matrix}
         1       & 1       & .. & 1             & 1 \\
         p_{i_0} & p_{i_1} & .. & p_{i_{\ds-1}} & p_{i_\ds}
       \end{matrix}
     \right]
     =
     \det\left[
       \begin{matrix}
         1 & 1       & .. & 1             & 1 \\
         0 & p_{i_1} & .. & p_{i_{\ds-1}} & \Delta \cdot e_1
       \end{matrix}
     \right]
     =
     \Delta \cdot
     \det\left[
       \begin{matrix}
         1 & 1            & .. & 1 \\
         0 & \bar p_{i_1} & .. & \bar p_{i_{\ds-1}}
       \end{matrix}
     \right].
  \end{align*}
  Here $\bar p_j \in \R^{\ds-1}$ is obtained by deleting the first coordinate of $p_j$,
  $e_1$ is the first canonical coordinate vector in $\R^\ds$
  and $\Delta := (p_0 - p_n)_1 = \langle S(X), x_1 \rangle$ is the total increment of $X$ in the $x_1$ direction.
  Here we used that $\ds$ is odd (otherwise we would get a prefactor $-1$).

  This is the expression for the summands of the right-hand side of \eqref{eq:invDet},
  with dimension $\ds -1$ and points $0 = \bar p_0, \bar p_1, .., \bar p_{n-1}$.
  %\footnote{
  %This is not true for $\ds$ even: there, the sequences used in a dimension below, i.e. in odd dimension,
  %always contain the endpoint (which we deleted here).}
  %
  By assumption, summing up all these determinants gives
  \begin{align*}
    \Delta \cdot \Big\langle S(\bar X), \firstInvariant{\ds-1} \Big\rangle
    =
    \Big\langle S(X), x_1 \Big\rangle \Big\langle S(\bar X), \firstInvariant{\ds-1} \Big\rangle,
  \end{align*}
  where $\bar X$ is the curve in $\R^{\ds-1}$ through the points $\bar p_0, .. \bar p_{n-1}$.
  %Since $p_0 = 0$ and since $p_n$ lies on the $x_1$ axis
  Since $\bar p_n = \bar p_0 = 0$,
  we can attach the additional point $\bar p_n$ to $\bar X$ without changing the value here (Lemma~\ref{lem:lastPoint}).
  Hence the sum of determinants is equal to
  \begin{align*}
    \Big\langle S(X), x_1 \Big\rangle \Big\langle S(X), \firstInvariant{\ds-1}(x_2,..,x_\ds) \Big\rangle,
  \end{align*}
  By Lemma~\ref{lem:dxw1} this is equal to $\langle S(X), \firstInvariant{\ds} \rangle$, which finishes the proof for odd $\ds$.

  ~\\

  \textbf{$\ds$ is even }

  We proceed by induction on $n$.
  For $n=\ds$ the statement follows from Lemma~\ref{lem:invariantVolume}.

  Let it be true for some $n$.
  Write $X = X'' \sqcup X'$ where $X'$ is the linear interpolation
  of $p_0, .., p_n$, $X''$ is the linear path from $p_n$ to $p_{n+1}$
  and we recall concatenation $\sqcup$ of paths from Lemma \ref{lem:chen}.
  Adding an additional point $p_{n+1}$,
  the sum on the right hand side of \eqref{eq:invDet}
  gets additional indices of the form
  \begin{align*}
    (p_{j_0}, .., p_{j_{d-1}}, p_{n+1}),
  \end{align*}
  where
  \begin{align*}
    j_0 &= 0 \\
    j_{\ds-1} &= n,
  \end{align*}
  and $j_1,..,j_{\ds-2}$ ranges over all possible increasing
  subsequences of $1,2,..,n-1$ such that
  for $\ell$ odd $j_\ell + 1 = j_{\ell+1}$.

  %Now, by $GL(\R^\ds)$ invariance we can assume that $p_{n+1} - p_n = (c,0,0,..,0)$ lies on the $x_1$ axis.
  Assume $p_{n+1} - p_n = \Delta \cdot e_1$ lies on the $x_1$-axis.
  Then, summing over those $j$,
  \begin{align*}
    \sum_j
     \det\left[
       \begin{matrix}
         1 & 1 & .. & 1 & 1 & 1 \\
         0 & p_{j_1} & .. & p_{j_{\ds-2}} & p_{n} & p_{n+1}
       \end{matrix}
     \right]
    %&=
    %\sum_j
    % \det\left[
    %   \begin{matrix}
    %     1 & 1 & .. & 1 & 1 & 1 \\
    %     0 & p_{j_1} & .. & p_{j_{\ds-2}} & p_{j_{\ds-1}} & \Delta \cdot e_1
    %   \end{matrix}
    % \right] \\
    % &=
    % \Delta 
    % \cdot
    % \sum_j
    % \det\left[
    %   \begin{matrix}
    %     1 & 1            & .. & 1 \\
    %     0 & \bar p_{j_1} & .. & \bar p_{j_{\ds-1}}
    %   \end{matrix}
    % \right].
    &=
    \sum_j
     \det\left[
       \begin{matrix}
         1             & 1             & .. & 1 & 1 & 1 \\
         - p_n & p_{j_1} - p_n & .. & p_{j_{d-2}} - p_n & 0 & p_{n+1} - p_n
       \end{matrix}
     \right] \\
    &=
    \sum_j
     \det\left[
       \begin{matrix}
         1             & 1             & .. & 1 & 1 & 1 \\
         - p_n & p_{j_1} - p_n & .. & p_{j_{d-2}} - p_n & 0 & \Delta \cdot e_1
       \end{matrix}
     \right] \\
    &=
    \Delta \cdot
    \sum_j
     \det\left[
       \begin{matrix}
         1             & 1             & .. & 1 & 1 \\
         - \bar p_n & \bar p_{j_1} - \bar p_n & .. & \bar p_{j_{d-2}} - \bar p_n & 0
       \end{matrix}
     \right] \\
    &=
    \Delta \cdot
    \sum_j
     \det\left[
       \begin{matrix}
         1             & 1             & .. & 1 & 1 \\
         0 & \bar p_{j_1} & .. & \bar p_{j_{d-2}} & \bar p_n
       \end{matrix}
     \right] \\
     &=
    \Delta
    \cdot
    \Big\langle S(X'), \firstInvariant{\ds-1}(x_2, .., x_\ds) \Big\rangle
  \end{align*}
  Here we used that the indices $j$ range over the ones used in dimension $\ds-1$
  on the points $\bar p_0, .., \bar p_n$.

  On the other hand,
  \begin{align*}
    \Big\langle S(X), \firstInvariant{\ds} \Big\rangle 
    &=
    \Big\langle S(X'') S(X'), \firstInvariant{\ds} \Big\rangle \\
    &=
    \Big\langle S(X'), \firstInvariant{\ds} \Big\rangle
    +
    \Big\langle S(X''), x_1 \Big\rangle
    \Big\langle S(X'), \firstInvariant{\ds-1}(x_2, .., x_\ds) \Big\rangle.
  \end{align*}
  Here we used that $S(X'') = \exp( \Delta \cdot x_1 ) = 1 + \Delta \cdot x_1 + O(x_1^2)$
  and that each monomial in $\firstInvariant{\ds}$ has exactly one occurrence of $x_1$.
  This finishes the proof.
\end{proof}

\begin{definition}	
	Let $X: [0,T] \to \R^\ds$ be any curve.
	Define its \textbf{signed volume} to be the following limit, if it exists,
	\begin{align*}
	  \operatorname{Signed-Volume}(X)
	  :=
    \frac{1}{\ds!}
    \lim_{|\pi| \to 0}
    \sum_i
     \det\left[
       \begin{matrix}
         1 & 1 & .. & 1 \\
         %X_{t^\pi_{i_0}} & p_{i^\pi_1} & .. & p_{i^\pi_{\ds}}
         X_{t^\pi_{i_0}} & X_{t^\pi_{i_1}} & .. & X_{t^\pi_{i_\ds}}
       \end{matrix}
     \right].
	\end{align*}
  Here $\pi = (0=t^\pi_0, .., t^\pi_{n^\pi}=T)$ is a partition of the interval $[0,T]$ and $|\pi|$ denotes its mesh size.
  The indices $i$ are chosen as in Lemma~\ref{lem:generalPiecewiseLinearCurve}.
\end{definition}

\begin{theorem}
	Let $X: [0,T] \to \R^\ds$ a continuous curve of bounded variation.
	Then its signed volume exists and
	\begin{align*}
    \operatorname{Signed-Volume}(X)
    =
    \frac{1}{\ds !}
    \Big\langle S(X)_{0,T}, \firstInvariant{\ds} \Big\rangle
	\end{align*}
\end{theorem}
\begin{proof}
	Fix some sequence $\{\pi^n\}_{n\in\N}$,
  of partitions of $[0,T]$ with $|\pi^n| \to 0$ and interpolate 	$X$ linearly along each $\pi^n$ to obtain a sequence of linearly interpolated curves $X^n$.
	Then by Lemma~\ref{lem:generalPiecewiseLinearCurve}
	\begin{align*}
    \operatorname{Signed-Volume}(X^n)
    =
    \frac{1}{\ds!}
    \Big\langle S(X^n)_{0,T}, \firstInvariant{\ds} \Big\rangle
	\end{align*}
	By stability of the signature in the class of continuous curves of bounded variation (\cite[Proposition 1.28, Proposition 2.7]{bib:FV2010}),
  we get convergence
	\begin{align*}
    \Big\langle S(X^n)_{0,T}, \firstInvariant{\ds} \Big\rangle
    \to
    \Big\langle S(X)_{0,T}, \firstInvariant{\ds} \Big\rangle
	\end{align*}
	and this is independent of the particular sequence $\pi^n$ chosen.
\end{proof}

The previous theorem is almost a tautology, but there are relations to classical objects in geometry.
For $\ds=2$, as we have seen in Section~\ref{sec:d2w1},
\begin{align*}
  \frac{1}{2} \Big\langle S(X)_{0,T}, \firstInvariant{2} \Big\rangle,
\end{align*}
is equal to the signed area of the curve $X$.
In general dimension, the value of the invariant is related
to some kind of classical ``volume'' if the curve satisfies
some kind of monotonicity.
This is in particular satisfied for the ``moment curve''.
\begin{lemma}
  \label{lem:momentCurve}
  Let $X$ be the moment curve
  \begin{align*}
    X_t = (t,t^2,...,t^\ds) \in \R^\ds.
  \end{align*}
  Then for any $T > 0$
  \begin{align*}
    \frac{1}{\ds !} \Big\langle S(X)_{0,T}, \firstInvariant{\ds} \Big\rangle
    =
    |\operatorname{Convex-Hull}( X_{[0,T]} )|
  \end{align*}
\end{lemma}
\begin{remark}
  It is easily verified that for integers $n_1 .. n_\ds$ one has
  \begin{align*}
    \frac{1}{n_1 \cdot .. \cdot n_\ds}
    \int_0^T dt_1^{n_1} .. dt_\ds^{n_\ds}
    =
    \frac{1}{n_1}
    \frac{1}{n_1+n_2}
    ..
    \frac{1}{n_1+..+n_\ds}
    T^{n_1+..+n_\ds}.
  \end{align*}
  We deduce that
  \begin{align*}
    |\operatorname{Convex-Hull}( X_{[0,T]} )|
    =
    T^{1+2+..+\ds}
    \sum_{\sigma\in S_\ds} \sign \sigma
    \frac{1}{\sigma(1)}
    \frac{1}{\sigma(1)+\sigma(2)}
    ..
    \frac{1}{\sigma(1)+..+\sigma(\ds)}.
  \end{align*}

  In \cite[Section 15]{bib:KS1953}, 
  the value of this volume is determined,
  for $T=1$, as
  \begin{align*}
    \prod_{\ell=1}^\ds \frac{(\ell-1)! (\ell-1)!}{((2\ell-1)!}.
  \end{align*}

  We hence get the combinatorial identity
  \begin{align*}
    \prod_{\ell=1}^\ds \frac{(\ell-1)! (\ell-1)!}{(2\ell-1)!}
    =
    \sum_{\sigma\in S_\ds}
    \sign \sigma
    \frac{1}{\sigma(1)}
    \frac{1}{\sigma(1)+\sigma(2)}
    ..
    \frac{1}{\sigma(1)+..+\sigma(\ds)}.
  \end{align*}
\end{remark}
\begin{proof}
  For $n\ge \ds$ let $0 = t_0 <  .. < t_n \le T$ be time-points, let $p_i := X_{t_i}$ be
  the corresponding points on the moment curve
  and denote by $X^n$ the piecewise linear curve through those points.
  We will show
  \begin{align*}
    \frac{1}{\ds !} \Big\langle S(X^n)_{0,T}, \firstInvariant{\ds} \Big\rangle
    =
    |\operatorname{Convex-Hull}( X^n_{[0,T]} )|.
  \end{align*}

  The convex hull of the point $\{p_i\}$ (equivalently: the convex hull of $X^n$)
  is known as the \emph{cyclic polytope} $C_\ds(n)$ \cite[Section 15.5.1.4]{bib:TOG2004}.
  A \emph{triangulation of a polytope} in dimension $\ds$ concerns its 
  (disjoint, up to to measure zero) decomposition
  into simplices of dimension $\ds$ (\cite[Chapter 16]{bib:TOG2004}).
  In particular $\{p_0, .., p_n \} = \cup_\ell S_\ell$ with $|S_\ell| = \ds+1$
  and
  \begin{align*}
    |\operatorname{Convex-Hull}( p_0,..,p_n )|
    =
    \sum_\ell |\operatorname{Convex-Hull}( S_\ell )|.
  \end{align*}

  We will show that the index sets summed over in %on the right-hand side of \eqref{eq:invDet}
  Lemma~\ref{lem:generalPiecewiseLinearCurve} form
  a certain kind of triangulation for $p_0,..,p_n$.

  The \emph{pulling triangulation} of $p_0,..,p_n$ with respect to $p_0$ is formed as follows:
  form all subset $p_0 \cup I$ where
  $I$ ranges over all $\ds$ points in $p_1,..,p_n$ such that they form
  a $\ds-1$-dimensional face (a facet) of the cyclic polytope.
  For any polytope, successively pulling each vertex results
  in a triangulation, see \cite[Chapter 16]{bib:TOG2004}. % \cite{bib: }
  For the polytope under consideration it is sufficient to pull one vertex, since
  all vertices lie on the boundary of the convex hull.

  By Gale's evenness criterion (\cite[Theorem 3]{bib:Gal1963})
  the points $p_{i_1}, .., p_{i_\ds}$,
  with distinct $i_j \in \{0,..,n\}$ form a facet if and only if 
  any two elements of $\{0,..,n\} \setminus \{i_1, .., i_\ds\}$
  are separated by an even number of elements in $\{i_1, .., i_\ds\}$.%
  \footnote{
    % 0,1,2,3,4
    For example,
    with $n=4$ and dimension $\ds=2$,
    the indices $\{0,1,2\}$,$\{0,2,3\}$,$\{0,3,4\}$,$\{0,1,4\}$,$\{1,2,4\}$,$\{2,3,4\}$
    lead to facets.
  }
  
  \textbf{$\ds$ odd}\\
  For the pulling triangulation, we are looking for such $\{i_j\}$ such that $i_1 \ge 1$.
  Those are exactly the indices with
  \begin{itemize}
    \item $i_{\ell+1} = i_\ell + 1$ for $\ell$ odd
    \item $i_\ds = n$.
  \end{itemize}
  Together with $i_0 := 0$ these form the indices of Lemma~\ref{lem:generalPiecewiseLinearCurve}.

  \textbf{$\ds$ even}\\
  We are looking for such $\{i_j\}$ such that $i_1 \ge 1$.
  Those are exactly the indices with
  \begin{itemize}
    \item $i_{\ell+1} = i_\ell + 1$ for $\ell$ odd.
  \end{itemize}
  Together with $i_0 := 0$ these form the indices of Lemma~\ref{lem:generalPiecewiseLinearCurve}.

  Hence
  \begin{align*}
    |\operatorname{Convex-Hull}( X^n_{[0,T]} )|
    =
    \sum_i |\operatorname{Convex-Hull}( p_{i_0}, .., p_{i_\ds} )|.
  \end{align*}
  Now by Lemma~\ref{lem:relationToClassicalVolume}
  \begin{align*}
    |\operatorname{Convex-Hull}( p_{i_0}, .., p_{i_\ds} )|
    =
    \frac{1}{\ds!}
    \left|
    \det\left[
      \begin{matrix}
        1 & 1 & .. & 1 \\
        p_{i_0} & p_{i_1} & .. & p_{i_\ds}
      \end{matrix}
    \right]
    \right|.
  \end{align*}
  The determinant is in fact positive here, since
  it is a Vandermonde determinant and can be written as
  \begin{align*}
    \prod_{0\le \ell < k \le n} ( t_{i_k} - t_{i_\ell} ) > 0.
  \end{align*}
  We can hence omit the modulus and get
  \begin{align*}
    |\operatorname{Convex-Hull}( p_{i_0}, .., p_{i_\ds} )|
    &=
    \frac{1}{\ds!}
    \det\left[
      \begin{matrix}
        1 & 1 & .. & 1 \\
        p_{i_0} & p_{i_1} & .. & p_{i_\ds}
      \end{matrix}
    \right] \\
    &=
    \Big\langle S(X^n)_{0,T}, \firstInvariant\ds \Big\rangle,
  \end{align*}
  by Lemma~\ref{lem:generalPiecewiseLinearCurve}.

  The statement of the lemma now follows by piecewise linear approximation of $X$
  using continuity of the convex hull, which follows from \cite[Lemma 3.2]{bib:EN2010},
  and of iterated integrals \cite[Proposition 1.28, Proposition 2.7]{bib:FV2010}.

  % calculation of $inv_d$ should be possible in closed form.
  % compare \emph{Engtrom,Noren - Polytopes from Subgraph Statistics}.
  % and \emph{Karlin,Shapley - Geometry of moment spaces - Section 15}
\end{proof}

\section{Rotations}
\label{sec:so}

Let
\begin{align*}
  SO(\R^\ds) = \{ A \in GL(\R^\ds) : A A^\top = \operatorname{id}, \det( A ) = 1 \},
\end{align*}
be the group of rotations of $\R^\ds$.

\begin{definition}
  We call $\phi \in T(\R^\ds)$ an \textbf{SO invariant} if
  \begin{align*}
    \langle S(X)_{0,T}, \phi \rangle = \langle S(A X)_{0,T}, \phi \rangle
  \end{align*}
  for all $A \in SO(\R^\ds)$ and all curves $X$.

  Alternatively, as explained in Section~\ref{sec:gl},
  \begin{align*}
    A^\top \phi = \phi,
  \end{align*}
  for all $A \in SO(\R^\ds)$, where the action on $\TS$ was given in Definition~\ref{def:action}.
\end{definition}

Since $\det(X) = 1$, any $GL$ invariant of weight $w \ge 1$ (Section~\ref{sec:gl})
is automatically an $SO$ invariant. But there are $SO$ invariants that are not $GL$ invariants (of any weight),
for example, for $d=2$, $\phi := x_1 x_1 + x_2 x_2$.

Switching to the perspective on multilinear maps, this is the map $(v_1,v_2) \mapsto \langle v_1, v_2 \rangle$.
It is shown, see for example {\cite[Theorem 2.9.A]{bib:Wey1946}},
that all invariants are built from the inner product and the determinant.

Recently, a linear basis for these invariants has been constructed.
To formulate the result, we need to introduce some notation from \cite{bib:LK2007}.
Define % p.164
\begin{align*}
  I(r,n) := \{( i_1, .., i_r ) : 1 \le i_1 < .. < i_r \le n \}.
\end{align*}
Use the following partial order on these sequences: for $a \in I(r,n), b \in I(r',n)$
\begin{align*}
  a \ge b
\end{align*}
if $r \le r'$ and $a_j \ge a'_j$ for $j \le r$.

For $c \in I(\ds,n)$ and $v_1, .., v_n \in \R^\ds$, define % p.167
\begin{align*}
  u(c)( v_1,..,v_n ) := \text{ $\ds$-minor of the $m \times n$ matrix $(v_1,..,v_n)$, with columns given by $c$ }.
\end{align*}

For $a,b \in I(r,n) \times I(r,n)$ with $r \le \ds$ and $v_1,..,v_n \in \R^\ds$, define
\begin{align*}
  p(a,b)( v_1,..,v_n ) := \text{ $r$-minor of the matrix $\langle v_i, v_j \rangle$, rows given by $a$, columns given by $b$ } 
\end{align*}

\begin{theorem}[{\cite[Theorem 12.5.0.8]{bib:LK2007}}] 
  Let $V$ be a $\ds$-dimensional vector space with inner product $\langle \cdot, \cdot \rangle$.
  A basis for the space of multilinear maps
  \begin{align*}
    \psi: \underbrace{V \times \dots \times V}_{n \text{ times}} \to \R
  \end{align*}
  that satisfy
  \begin{align*}
    \psi(A v_1, A v_2, \dots, A v_n) = \psi(v_1, v_2, \dots, v_n)
  \end{align*}
  for all $A \in \operatorname{SO}(V)$ and $v_1, \dots, v_n \in V$ 
  is given by the maps
  \begin{align*}
    F = p\left(a^{(1)},b^{(1)}\right) \cdot .. \cdot p\left(a^{(r)},b^{(r)}\right) u\left(c^{(1)}\right) \cdot .. \cdot u\left(c^{(s)}\right),
  \end{align*}
  with $c^{(j)} \in I(\ds,n)$,
  $a^{(j)},b^{(j)} \in I(r,n), 1 \le r \le \ds-1$,
  \begin{align*}
    a^{(1)} \ge b^{(1)} \ge a^{(2)} \ge .. \ge b^{(r)} \ge c^{(1)} \ge .. \ge c^{(s)},
  \end{align*}
  and
  \begin{align*}
    \cup_j a^{(j)} \bigcup \cup_j b^{(j)} \bigcup \cup_j c^{(j)} = \{1, .., n\},
  \end{align*}
  is a \emph{disjoint} union (that is, every number $1,..,n$ appears in exactly one of the sequences $a^{(j)},b^{(j)},c^{(j)}$).
  In particular $n = C_1 \cdot 2 + C_2 \cdot \ds$ for some $C_1,C_2 \in \N$.
\end{theorem}

\begin{example}
  $\ds=2$

  $n=1$:
  There is no such set of sequences.

  $n=2$:
  Allowed sets of sequences are
  \begin{itemize}
    \item $c^{(1)} = (1,2)$\\
        $\leadsto F(v_1,v_2) = \langle v_1, v_2 \rangle$
    \item $a^{(1)} = (2), b^{(1)} = (1)$\\
      $\leadsto F(v_1,v_2) = \det[ v_1 v_2 ]$
  \end{itemize}

  $n=3$:
  There is no such set of sequences.

  $n=4$:
  Allowed sets of sequences are
  \begin{itemize}
    \item $a^{(1)} = (4), b^{(1)} = (3), a^{(2)} = (2), b^{(2)} = (1)$\\
      $\leadsto F(v_1,v_2,v_3,v_4) = \langle v_4, v_3 \rangle \langle v_2, v_1 \rangle$

    \item $a^{(1)} = (4), b^{(1)} = (3), c^{(1)} = (1,2)$\\
      $\leadsto F(v_1,v_2,v_3,v_4) = \langle v_4, v_3 \rangle \det[ v_1 v_2 ]$

    \item $a^{(1)} = (4), b^{(1)} = (2), c^{(1)} = (1,3)$

    \item $a^{(1)} = (3), b^{(1)} = (2), c^{(1)} = (1,4)$

    \item $c^{(1)} = (3,4), c^{(2)} = (1,2)$

    \item $c^{(1)} = (2,4), c^{(2)} = (1,3)$
  \end{itemize}
\end{example}

In the setting of $\TS$ we have
\begin{proposition}
  \label{prop:soInvariants}
  The $SO$ invariants of homogeneity $n$ are spanned by
  \begin{align*}
    \bijection( \Psi ),
  \end{align*}
  where $\Psi$ ranges over the invariants of the previous theorem
  and $\bijection$ is given in Lemma~\ref{lem:oneToOne}.
\end{proposition}

In the case $\ds=2$, there is another way to arrive at a basis for the invariants.
%even giving an explicit basis.
%
Taking inspiration from \cite{bib:Flu2000}, which concerns rotation invariants
of images, we work in the complex vector space $T(\C^2)$.
\begin{theorem}
  \label{thm:soInvariants}
  Define
  \begin{align*}
    z_1 &= x_1 + i x_2 \\
    z_2 &= x_1 - i x_2.
  \end{align*}
  The space of $SO$ invariants on level $n$ in $T(\C^2)$ is spanned freely by
  \begin{align*}
    z = z_{j_1} \cdot .. \cdot z_{j_n} \quad \text{ with } \quad \#\{ r : j_r = 1 \} = \#\{ r : j_r = 2 \}.
  \end{align*}

  The space of $SO$ invariants on level $n$ in $T(\R^2)$ is spanned freely by
  \begin{align*}
    \operatorname{Re}[ z ], \operatorname{Im}[ z ] \quad \text{ with } \quad \#\{ r : j_r = 1 \} = \#\{ r : j_r = 2 \} \text{ and } z_1 = 1.
  \end{align*}
  %All invariants of homogeneity $n$ in $T(\R^2)$ are spanned
  %by taking imaginary and real part of these complex invariants.
\end{theorem}
\begin{remark}
  In particular for $\ds=2$ and $n$ even, the dimension of rotation invariants on level $n$ in $T(\R^2)$
  is equal to
  %$2\ {n-1 \choose n/2-1}$.
  $\binom{n}{n/2}$.
\end{remark}
\begin{proof}
  1. \textbf{$z$ is invariant}\\
  Let
  \begin{align*}
    A_\theta :=
    \begin{pmatrix}
      \cos(\theta) & \sin(\theta) \\
      -\sin(\theta) & \cos(\theta)
    \end{pmatrix}
  \end{align*}

  Then (recall Definition~\ref{def:action})
  \begin{align*}
    A_\theta^\top z_1
    &=
    A_\theta^\top (x_1 + i x_2) \\
    &=
    \cos(\theta) x_1 + \sin(\theta) x_2
    +
    i \left( -\sin(\theta) x_1 + \cos(\theta) x_2 \right) \\
    &=
    e^{-i\theta} z_1 \\
    A_\theta^\top z_2 &= e^{i\theta} z_2.
  \end{align*}
  Hence
  \begin{align*}
    A_\theta^\top z_{j_1} \cdot .. \cdot z_{j_n} = z_{j_1} \cdot .. \cdot z_{j_n} \forall \theta
    \quad
    \text{ if and only }
    \quad
    \#\{ r : j_r = 1 \} = \#\{ r : j_r = 2 \}.
  \end{align*}

  2. \textbf{They form a basis}\\
  Now $x_{j_1} .. x_{j_n}: j_\ell \in \{1,2\}$ is a basis of $\pi_n T(\C^2)$ with respect to $\C$.
  Hence $z_{j_1} .. z_{j_n}$ is (the map $(x_1,x_2) \mapsto (z_1,z_2)$ is invertible).
  By Step 1 we have hence exhibited a basis (with respect to $\C$) for all invariants in $\pi_n T(\C^2)$.

  3. \textbf{Real invariants}\\
	The space of $SO$ invariants on level $n$ in $T(\C^2)$ is spanned freely by the set of
	\begin{align*}
	z_{j_1} \cdot .. \cdot z_{j_n} \quad \text{ with } \quad \#\{ r : j_r = 1 \} = \#\{ r : j_r = 2 \}.
	\end{align*}
	Adding and subtracting the elements with $j_1=2$ from the elements with $j_1=1$, we get that the space of $SO$ invariants on level $n$ in $T(\C^2)$ is spanned freely by the set of
	\begin{align*}
	 & (z_{j_1} \cdot .. \cdot z_{j_n} + z_{3-j_1} \cdot .. \cdot z_{3-j_n}) \quad\text{and}\quad (z_{j_1} \cdot .. \cdot z_{j_n} - z_{3-j_1} \cdot .. \cdot z_{3-j_n}) 
	 \\&\quad \text{ with } \quad \#\{ r : j_r = 1 \} = \#\{ r : j_r = 2 \}\text{ and $j_1=1$}.
	\end{align*}
	Because $z_{3-j_1} \cdot .. \cdot z_{3-j_n}$ is the complex conjugate of $z_{j_1} \cdot .. \cdot z_{j_n}$, this means that the space of $SO$ invariants on level $n$ in $T(\C^2)$ is spanned freely by the set of
		\begin{align*}
	& \operatorname{Re}(z_{j_1} \cdot .. \cdot z_{j_n}) \quad\text{and}\quad \operatorname{Im}(z_{j_1} \cdot .. \cdot z_{j_n}) 
	\\&\quad \text{ with } \quad \#\{ r : j_r = 1 \} = \#\{ r : j_r = 2 \}\text{ and $j_1=1$}.
	\end{align*}
	This is an expression for a basis of the SO invariants in terms of real combinations of basis elements of the tensor space. They thus form a basis for the SO invariants for the free \emph{real} vector space on the same set, namely $\pi_n T(\R^2)$.
\end{proof}

\begin{example}
  \label{ex:soInvariants}
  Consider $\ds=2$

  Order $2$
  \begin{align*}
  \word{11} + \word{22} \\
 - \word{12} + \word{21}
  \end{align*}

  Order $4$
  \begin{align*}
  \word{1111} - \word{1122} + \word{1212} + \word{1221} + \word{2112} + \word{2121} - \word{2211} + \word{2222} \\
 - \word{1112} - \word{1121} + \word{1211} - \word{1222} + \word{2111} - \word{2122} + \word{2212} + \word{2221} \\
  \word{1111} + \word{1122} - \word{1212} + \word{1221} + \word{2112} - \word{2121} + \word{2211} + \word{2222} \\
 - \word{1112} + \word{1121} - \word{1211} - \word{1222} + \word{2111} + \word{2122} - \word{2212} + \word{2221} \\
  \word{1111} + \word{1122} + \word{1212} - \word{1221} - \word{2112} + \word{2121} + \word{2211} + \word{2222} \\
  \word{1112} - \word{1121} - \word{1211} - \word{1222} + \word{2111} + \word{2122} + \word{2212} - \word{2221}
  \end{align*}

   Consider $\ds=3$

   Order $3$
   \begin{align*}
   \word{123} - \word{132} + \word{312} - \word{321} + \word{231} - \word{213}
   \end{align*}

   Consider $\ds=4$

   Order $2$
   \begin{align*}
    \word{11} + \word{22} + \word{33} + \word{44}.
   \end{align*}

   Order $4$
   \begin{align*}
     &\word{1144} + \word{4422} + \word{4444} + \word{3333} + \word{1122} + \word{4433} + \word{1133} + \word{4411} + \word{2211} + \word{3344} + \word{1111}\\
     &\qquad + \word{2244} + \word{2222} + \word{3322} + \word{2233} + \word{3311} \\
     &\word{4343} + \word{3232} + \word{3131} + \word{4444} + \word{3333} + \word{4242} + \word{2121} + \word{1212} + \word{2323} + \word{4141} + \word{2424}\\
     &\qquad + \word{1313} + \word{1111} + \word{3434} + \word{1414} + \word{2222} \\
     &\word{4334} + \word{2332} + \word{3223} + \word{4444} + \word{3333} + \word{1441} + \word{2442} + \word{1221} + \word{1331} + \word{3113} + \word{3443}\\
     &\qquad + \word{1111} + \word{2112} + \word{4114} + \word{2222} + \word{4224} \\
     &\word{1423} - \word{1432} - \word{1324} - \word{4231} + \word{3412} - \word{2341} - \word{1243} - \word{4123} - \word{2413} + \word{1234} + \word{3124}\\
     &\qquad + \word{4213} + \word{1342} + \word{2431} - \word{2134} + \word{2143} + \word{3241} + \word{2314} + \word{4321} - \word{4312} - \word{3214}\\&\qquad - \word{3142}
     - \word{3421} + \word{4132}
   \end{align*}
\end{example}

\section{Permutations}
\label{sec:permuations}

Denote by $S_\ds$ the group of permutations of $[d] := \{1, .., \ds\}$.
\begin{lemma}
  For $\sigma \in S_\ds$,
  define $M(\sigma) \in GL(\R^d)$ as
  \begin{align*}
    M(\sigma)_{ij} = 1 \qquad \text{ if } i = \sigma(j).
  \end{align*}

  Then $M: S_\ds \to GL(\R^\ds)$ is a group homomorphism
  and
  moreover $M(\sigma^{-1}) = M(\sigma)^\top$.%
  \footnote{$M$ is sometimes called the \emph{defining representation} of $S_d$.}
\end{lemma}
\begin{proof}
  Regarding the first point, for $i=\{1,..,\ds\}$,
  \begin{align*}
    M(\sigma) M(\tau) e_i
    =
    M(\sigma) e_{\tau(i)}
    =
    e_{\sigma(\tau(i))}
    =
    M(\sigma \tau) e_i.
  \end{align*}

  Regarding the last point
  \begin{align*}
    M_{ij} &= 1 \qquad \text{ if } i = \sigma^{-1}(j) \\
      &\Leftrightarrow \\
    M_{ij} &= 1 \qquad \text{ if } \sigma( i ) = j \\
      &\Leftrightarrow \\
    M_{ij} &= 1 \qquad \text{ if }  j = \sigma( i ).
  \end{align*}
\end{proof}

$S_\ds$ then acts on $\TC$ and $\TS$ via Definition~\ref{def:action}.
Explicitly,
\begin{align*}
  \sigma \cdot x_{i_1} .. x_{i_n} = x_{\sigma(i_1)} .. x_{\sigma(i_n)}.
\end{align*}
%Indeed,
%\begin{align*}
%  \sigma \cdot x_{i_1} .. x_{i_n} 
%  &=
%  M(\sigma) \cdot x_{i_1} .. x_{i_n} \\
%  &=
%  M_{j_1 i_1}(\sigma) x_{j_1} .. M_{j_n i_n} x_{j_n} \\
%  &=
%  x_{\sigma(i_1)} .. x_{\sigma(i_n)}.
%\end{align*}

%\begin{example}
%  $\sigma = (2,3,1)$
%\end{example}

\begin{definition}
  We call $\phi \in T(\R^\ds)$ a \textbf{permutation invariant} if
  \begin{align*}
    \Big\langle S( M(\sigma) X)_{0,T}, \phi \Big\rangle = \Big\langle S(X)_{0,T}, \phi \Big\rangle
  \end{align*}
  for all $\sigma \in S_\ds$ and all curves $X$.
  Alternatively, as explained in Section~\ref{sec:gl},
  \begin{align*}
    M(\sigma)^\top \phi = \phi,
  \end{align*}
  for all $\sigma \in S_\ds$.
  Equivalently,
  \begin{align*}
    M(\sigma) \phi = \phi,
  \end{align*}
  for all $\sigma \in S_\ds$,
\end{definition}

%\begin{remark}
%  An $SO$ invariant is a a permutation invariant, if we restrict to even permutations.
%\end{remark}
  
We follow \cite[Section 3]{bib:BRRZ2005}.
To a monomial
\begin{align*}
  x_{i_1}\cdot .. \cdot x_{i_n},
\end{align*}
we associate the following set partition  of $[n] := \{1, .., n \}$
\begin{align*}
  \nabla(  x_{i_1} \cdot .. \cdot x_{i_n} )
  :=
  \{ \{ \ell : i_\ell = p \} : p \in [d] \} \setminus \{ \{ \} \}.
\end{align*}

\begin{example}
  Let $d=3$, then
  \begin{align*}
    \nabla(  x_2 x_3 x_2 x_2 x_1 )
    =
    \{ \{1, 3, 4\}, \{2\}, \{5\} \}.
  \end{align*}
\end{example}

Note that for every permutation $\sigma \in S_\ds$,
\begin{align}
  \label{eq:permuation}
  \nabla( x_{i_1}\cdot .. \cdot x_{i_n} ) = \nabla( x_{\sigma(i_1)} \cdot .. \cdot x_{\sigma(i_n)} ).
\end{align}

\begin{proposition}[{\cite[Section 3]{bib:BRRZ2005}}]
  \label{prop:permInvariants}
  Define
  \begin{align*}
    M_A := \sum_{i: \nabla( x_{i_1} .. x_{i_n} ) = A} x_{i_1} .. x_{i_n}.
  \end{align*}
  Then $\{ M_A : A \text{ is set partition of } [n] \text{ and } |A| \le \ds \}$ is a linear basis for the space of permutation invariants of homogeneity $n$.
\end{proposition}
\begin{proof}
  By \eqref{eq:permuation}, each $M_A$ is permutation invariant.
  Moreover, since $|A| \le \ds$, $M_A$ is nonzero.

  For $A, A'$ distinct set partitions of $[n]$
  the monomials in $M_A$ and the monomials in $M_{A'}$ do not overlap.
  Hence the proposed basis is linearly independent.

  Now, if $\phi$ is permutation invariant
  and if for some $i, i'$,
  $\nabla( x_{i_1} .. x_{i_n} ) = \nabla( x_{i'_1} .. x_{i'_n} )$ then
  the coefficient of $x_i$ and $x_{i'}$ must coincide.
  Hence the proposed basis spans invariants of homogeneity $n$.
\end{proof}

\begin{example}
  \label{ex:permInvariants}
  Consider $\ds=3$

  Order $n=1$
  \begin{align*}
    \word{1} + \word{2} + \word{3}
  \end{align*}

  Order $n=2$
  \begin{align*}
    &\word{33} + \word{22} + \word{11}  \\
    &\word{32} + \word{31} + \word{23} + \word{21} + \word{13} + \word{12}
  \end{align*}

  Order $n=3$
  \begin{align*}
    &\word{333} + \word{222} + \word{111} \\
    &\word{332} + \word{331} + \word{223} + \word{221} + \word{113} + \word{112} \\
    &\word{323} + \word{313} + \word{232} + \word{212} + \word{131} + \word{121} \\
    &\word{322} + \word{311} + \word{233} + \word{211} + \word{133} + \word{122} \\
    &\word{321} + \word{312} + \word{231} + \word{213} + \word{132} + \word{123}
  \end{align*}
\end{example}

%\begin{lemma}
%  Let $\tilde S := \{ \sigma \in S_{0,1,..,\ds} : \sigma(0) = 0 \}$.
%
%  We call $\phi \in T(\R^{1+\ds})$ a \emph{permutation invariant modulo time} if
%  \begin{align*}
%    M(\sigma) \phi = \phi,
%  \end{align*}
%  for all $\sigma \in \tilde S$.
%
%  The following is a basis of the invariants of homogeneity $m$:
%  ..
%\end{lemma}

\section{An additional (time) coordinate}
\label{sec:time}

Assume now that $X = (X^0,X^1,..,X^\ds): [0,T] \to \R^{1+\ds}$.
Here $X^0$ plays a special role, in that we assume that it is not affected by the space transformations under consideration.
%Often $X^0(t) = t$ is just a component keeping track of time.

Adding an ``artificial'' $0$-th component, usually keeping track of time, $X^0_t := t$,
is a common trick to improve the expressiveness of the signature.
In particular, if such an $X^0$ is monotonically increasing, the enlarged curve $(X^0,X^1,..,X^\ds)$
never has any ``tree-like'' components (compare Section~\ref{sec:discussion}),
no matter what the original $(X^1,..,X^\ds)$ was.

Consider $GL$ invariants for the moment.
\begin{definition}
  Let
  \begin{align*}
    GL_0(\R^\ds) := \{ A \in GL(\R^{1+\ds}) : A e_0 = A^{-1} e_0 = e_0 \},
    %\widetilde{GL}(\R^\ds) := \{ A \in GL(\R^{1+\ds}) : A e_0 = A^{-1} e_0 = e_0 \},
  \end{align*}
  the space of invertible maps of $\R^{1+\ds}$ leaving the first direction unchanged.
  We call $\phi \in T(\R^{1+\ds})$ a \textbf{$\widetilde{GL}$ invariant of weight $w$} if
  \begin{align*}
    A^\top \phi = (\det A)^w \phi,
  \end{align*}
  for all $A \in GL_0(\R^\ds)$.
\end{definition}

Consider the $GL(\R^2)$ invariant of weight $1$
\begin{align*}
  x_1 x_2 - x_2 x_1.
\end{align*}
Since elements of $GL_0(\R^2)$ leave the variable $x_0$ unchanged,
a straightforward way to produce $GL_0$ invariants presents itself:
insert $x_0$ at the same position in every monomial.
For example
\begin{align*}
  x_1 x_0 x_2 - x_2 x_0 x_1
\end{align*}
is a $GL_0(\R^2)$ invariant of weight $1$.
We now formalize this idea and show that we get every $GL_0$ invariant this way.

~\\

\newcommand{\ins}{\mathsf{Insert}}
\newcommand{\remove}{\mathsf{Remove\ }}
Define the linear map $\remove$ of ``removing instances of $x_0$'' on monomials, as
\begin{align*}
  \remove x_{i_1} .. x_{i_m} := \prod_{\ell : i_\ell \not= 0} x_{i_\ell},
\end{align*}
so for example
\begin{align*}
  \remove x_0 x_1 x_1 x_0 x_3 &= x_1 x_1 x_3 \\
  \remove x_0 x_0 &= 1.
\end{align*}

Define for $U \subset [m]$ and $i = (i_1, .., i_m)$
\begin{align*}
  i|_U = (i_\ell : \ell = 1,.., m; \ell \in U).
\end{align*}

Define the linear map of restriction to $U$
on polynomials of order $m$ by defining on monomials
\begin{align*}
  %x_{i_1} .. x_{i_m}|_U = \prod_{\ell = 1,..., m: \ell \in U} x_{i_\ell},
  x_i|_U := x_{i|_U}
\end{align*}
so for example
\begin{align*}
  x_{i_1} x_{i_2} x_{i_3} |_{\{1,3\}} = x_{i_1} x_{i_3}.
\end{align*}

For $z = (z_1,..,z_{m+1}) \in \N^{m+1}$
denote by $\ins_z$ the linear operator
on polynomials of order $m$ by defining it on monomials
as follows.
For a monomial $x_{i_1} .. x_{i_m}$ of order $m$,
$\ins_z$ inserts
$z_1$ occurrences of $x_0$ before $x_{i_1}$,
$z_2$ occurrences of $x_0$ before $x_{i_2}$,
..,
$z_m$ occurrences of $x_0$ before $x_{i_m}$
and
$z_{m+1}$ occurrences of $x_0$ after $x_{i_m}$.
For example
\begin{align*}
  \ins_{(2,1,4)} x_1 x_2 = x_0 x_0 x_1 x_0 x_1 x_0 x_0 x_0 x_0.
\end{align*}

\begin{theorem}
  A basis for
  the space of $GL_0$ invariant of weight $w$,
  homogeneous of degree $m$ is given by the polynomials
  %\todonotes{ jd: this is an uncountable set .. }
  %\begin{align*}
  %\{ \phi : \remove \phi = \psi, {\color{red}\phi|_U = (x_0)^{|U|}}, U \subset [m], \psi \text{ $GL$ invariant } \}.
  %\end{align*}
  %\begin{align*}
  %  \{
  %  \ins_z \psi:
  %  0 \le n \le m;
  %  \psi \text{ $GL$ invariant of order $w$ and homogeneity $n$};
  %  z \in \N^{n+1} s.t. \sum_\ell z_\ell = m - n \}
  %\end{align*}
  \begin{align*}
    \ins_z \psi,
  \end{align*}
  with $0 \le n \le m$,
  $\psi$ ranges over the basis for $GL$ invariant of weight $w$ and homogeneity $n$ (Proposition \ref{prop:glInvariants})
  and
  $z \in \N^{n+1}$ such that $\sum_\ell z_\ell = m - n$.
\end{theorem}
\begin{proof}
  Let $n, \psi, z$ be as in the statement,
  then $\ins_z \psi$ is $GL_0$ invariant of weight $w$.
  Indeed:
  for $A_0 = \operatorname{diag}(1, A) \in GL_0(\R^\ds)$, with $A \in GL(\R^\ds)$,
  \begin{align*}
    A_0\ \ins_z \psi
    =
    \ins_z A \psi
    =
    (\det A)^w \ins_z \psi.
  \end{align*}

  On the other hand,
  let $\phi$ of order $m$ be a $GL_0$ invariant modulo time of weight $w$.
  Define for $U \subset [m]$
  \begin{align*}
    \phi^U := \sum_{i : i_\ell = 0, \ell \in U; i_j \not= 0, j \not\in U} \langle \phi, x_i \rangle x_i,
  \end{align*}
  which collects all monomials having $x_0$ exactly at the positions in $U$.
  Then
  \begin{align*}
    \phi = \sum_{U \subset [m]} \phi^U.
  \end{align*}

  Now, since $\phi$ is $GL_0$ invariant of weight $w$
  and since $GL_0$ leaves
  \begin{align*}
    \operatorname{span} \{ x_i : i_\ell = 0, \ell \in U; i_j \not= 0, j \not\in U \}
  \end{align*}
  invariant, we get that $\phi^U$ is $GL_0$ invariant of weight $w$.
  Clearly, there is $0 \le n \le m$ and $i \in \N^{n+1}$ such that
  \begin{align*}
    \ins_z \remove \phi^U = \phi^U.
  \end{align*}
  Lastly, $\remove \phi^U$ is $GL$ invariant,
  since for $A_0 = \operatorname{diag}(1, A) \in GL_0(\R^\ds)$, with $A \in GL(\R^\ds)$,
  \begin{align*}
    A \ \remove \phi^U = \remove A_0 \phi^U = (\det A_0)^w \remove \phi^U = (\det A)^w \remove \phi^U.
  \end{align*}
  Hence every invariant is in the span of the set given in the statement.
  They are linearly independent, and hence form a basis.
\end{proof}

The corresponding statements for rotations and permutations are completely analogous, so we omit stating them.

%\section{ Relation to existing work }
%
%\label{sec:relation}
%
%
%\TODO{ CONTINUE HERE }
%.. integral invariants \cite{bib:FKK2010} ..
%
%\begin{itemize}
%  \item i am pretty sure their invariants are not complete (which they conjecture on p.906)
%  \item i am not sure about their claim of independence on p.908
%\end{itemize}
%
%They construct invariants built from expressions of the form
%\begin{align*}
%  \int_0^T (X^1_{0,t} )^{\alpha_1} \cdot .. \cdot  ( X^\ds_{0,t} )^{\alpha_\ds} dX^i_t.
%\end{align*}
%By the shuffle identity (Lemma~\ref{lem:shuffleIdentity}),
%these expression, and polynomials in them, can be found in the signature of iterated integrals.
%%
%Hence, invariants like $I_1, I_2, I_3$ on p.913
%can be found in the invariants we propose.
%%
%For example
%\begin{align*}
%  I_2 = ..
%\end{align*}
%
%
%.. invariants based on curvature \cite{bib:COSTH1998} ..

\section{Discussion and open problems}
\label{sec:discussion}

We have presented a novel way to extract invariant features of $\ds$-dimensional curves,
based on the iterated-integrals signature.
We have identified all those features that can be written as a finite linear combination of terms in the signature.

There is a vast literature on invariants of, mostly $2d$, curves in the literature.
Among the techniques used, the method of ``integral invariants'' 
\cite{bib:FKK2010} is closest to our setting
(it has been used for example in \cite{bib:GMW2009} for character recognition).
In that work, for a curve $X: [0,T] \to \R^\ds$, $\ds=2,3$,
the building blocks for invariants are expressions of the form
\begin{align}
  \label{eq:expression}
  \int_0^T (X^1_r)^{\alpha_1} .. (X^\ds_r)^{\alpha_\ds} dX^i_r, \qquad i=1,..,\ds.
\end{align}
%By the shuffle identity (Lemma~\ref{lem:shuffleIdentity}) these values can be found in the signature $S(X)_{0,T}$.
%
Using an algorithmic procedure,
some invariants to certain subgroups of $G \subset GL(\R^\ds)$ are derived.
In particular for $\ds=2$ and $G=GL(\R^\ds)$ the following invariants are given
\begin{align*}
  I_1 &= \frac{1}{2} \int_0^T X^1_{0,r} dX^2_r - \frac{1}{2} X^1_{0,t} X^2_{0,t}  \\
  I_2 &= \int_0^T X^1_{0,r} X^2_{0,r} dX^2_r\ X^1_{0,t} - \frac{1}{2} \int_0^t (X^1_r)^2 dX^2_r\ X^2_{0,t} \\
  I_3 &= \int_0^T X^1_{0,r} (X^2_{0,r})^2 dX^2_r\ X^2_{0,T} - \int_0^T (X^1_{0,r})^2 X^2_{0,r} dX^2_r\ X^1_{0,T} X^2_{0,T}
          + \frac{1}{3} \int_0^T (X^1_{0,r})^3 dX^2_r\ X^2_{0,r} X^2_{0,r} \\
          &\qquad
          - \frac{1}{12} (X^1_{0,t})^3 (X^2_{0,t})^3.
  %&= \frac{1}{2} \int_0^t X^1_{0,r} dX^2_r - \frac{1}{2} \int_0^t X^2_{0,r} dX^1_r,
\end{align*}
By the shuffle identity (Lemma~\ref{lem:shuffleIdentity}), we can write these as $I_i = \langle S(X)_{0,T}, \phi_i \rangle$, with
\begin{align*}
  \phi_1 &:= \frac{1}{2} \word{12} - \frac{1}{2} \word{12} \\
  \phi_2 &:= \frac13 \word{1221}
                            +\frac13 \word{1212}
                            -\frac23 \word{1122}
                            +\frac13 \word{2121}
                            +\frac13 \word{2112}
                            -\frac23 \word{2211} \\
  \phi_3 &:= - \word{121212} - \word{211122} + \word{212121} + \word{221112} - \word{121221} + \word{122211} - \word{112212}
  \\&\qquad + \word{122112}
    - \word{211212} - \word{211221} - \word{121122} + \word{122121} -3 \word{222111} +3 \word{111222} 
    \\&\qquad+ \word{221121} + \word{212211}
     - \word{112122}
    + \word{212112} - \word{112221} + \word{221211}.
\end{align*}
One can easily check that these lie in the linear span of the invariants given in Proposition~\ref{prop:soInvariants} (or Theorem~\ref{thm:soInvariants}),
as expected.

We note that expressions of the form \eqref{eq:expression}
are \emph{not} enough to uniquely characterize a path.
Indeed, the following lemma gives a counterexample to the conjecture on p.906 in \cite{bib:FKK2010}
that ``signatures of non-equivalent curves are different''
(here, the ``signature'' of a curve means the set of expressions of the form \eqref{eq:expression}).
\begin{lemma}
  Consider the two closed curves $X^+$ and $X^-$ in $\R^2$, given for $t$ in $[0,2\pi]$ as%
  %\footnote{
  %  The images of both these curves coincide
  %  and form an algebraic curve, called the ``lemniscate of Gerono''.}
  \begin{align*}
    X^{\pm,1}_t &=\pm\cos t\\
    X^{\pm,2}_t &=\sin 2t.
  \end{align*}

  Then all the expressions \eqref{eq:expression} coincide on $X^+$ and $X^-$.
\end{lemma}

These curves both trace a figure called the \emph{lemniscate of Gerono} which is illustrated in Figure~\ref{fig:lemniscate}.
\begin{figure}
\begin{center}
\begin{tikzpicture}
\begin{axis}[
trig format plots=rad,
axis equal,
%    hide axis
]
\addplot [domain=0:2*pi, samples=200, black] ({cos(x)}, {sin(2*x)});
\draw[|-Latex] (axis cs:1.07,0)--(axis cs:1.07,0.3);
\node at (axis cs: 1.22,0.14) {$\displaystyle X^+$};
\draw[|-Latex] (axis cs:-0.93,0)--(axis cs:-0.93,0.3);
\node at (axis cs: -0.77,0.14) {$\displaystyle X^-$};
\end{axis}
\end{tikzpicture}
\end{center}	
	\caption{\label{fig:lemniscate}The lemniscate of Gerono. Traversing it once in each of the two directions indicated gives two distinct closed curves with distinct iterated-integral signatures, but which cannot be distinguished with the ``signature'' of \cite{bib:FKK2010}.}
\end{figure}
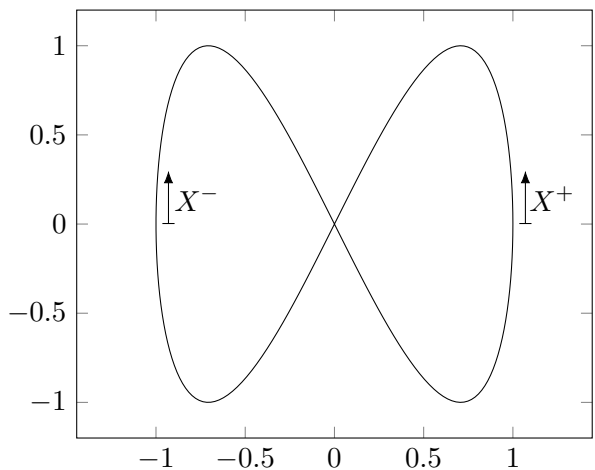

\begin{proof}
%  Let $m$ and $n$ be nonnegative integers.
%  Then
%  \begin{align*}
%    2 \int_0^{\pi/2} \cos^m t \sin^n t dt
%    &=
%    2 \int_0^{\pi/2} (\cos^2t )^{(m-1)/2} (\sin^2t)^{(n-1)/2} \, \cos(t)\sin(t) dt
%     \intertext{(letting $s=\sin^2t$, so $ds=2\cos t\sin t$)}%this avoids minus signs in comparison with s=cos^2t
%    &=
%    \int_0^{1} (1-s)^{(m-1)/2} s^{(n-1)/2} ds \\
%    &=
%    B(m/2+1/2,n/2+1/2),
%  \end{align*}
%  where $B$ is the beta function.\footnote{(Abramowitz \& Stegun 6.2.1)}

  %By repeated integration by parts
  %one sees that $\int_0^{2\pi}\cos^m t\;\sin^n t \,dt=0$  if $m$ or $n$ is odd.
  %\todonotes{IBP seems too complicated}
  
  Consider the function $f^m_n(t):=\cos^m t\;\sin^n t$, where $m$ and $n$ are nonnegative integers. If $n$ is odd, then $f^m_n(t)=-f^m_n(2\pi- t)$ so $\int_0^{2\pi}f^m_n(t)\,dt$ is zero. If $m$ is odd, then
  \begin{align*} \int_0^{2\pi}f^m_n(t)\,dt=-\int_{\frac\pi2}^{-\frac{3\pi}2}f^m_n(\frac\pi2-t)\,dt=\int^{\frac\pi2}_{-\frac{3\pi}2}f^n_m(t)\,dt=\int^{2\pi}_0f^n_m(t)\,dt=0.
\end{align*} 
%  Recall the identity $\sin(2t) = 2\sin t\ \cos t$.
  %
  Thus $\int_0^{2\pi}f^m_n(t)\,dt$ can only be nonzero if $m$ and $n$ are both even. 
  
  Any expression like \eqref{eq:expression} is either of the form
  \begin{align*}
  \int_0^{2\pi}x^my^n\,dx&=\int_0^{2\pi}(\pm1)^m\cos^m t\;\sin^n 2t \;(\mp\sin t)\,dt
  \\&=\mp2^n(\pm1)^m\int_0^{2\pi}\cos^{m+n}t\sin^{n+1}t\,dt
\\&=\begin{cases}0&\text{$n$ even or $m$ even}\\-2^n\int_0^{2\pi}\cos^{m+n}t\sin^{n+1}t\,dt&\text{otherwise}\end{cases}
  \end{align*}
  or of the form
  \begin{align*}
  \int_0^{2\pi}x^my^n\,dy&=\int_0^{2\pi}(\pm1)^m\cos^m t\;\sin^n 2t \;(2\cos t)\,dt
  \\&=2^{n+1}(\pm1)^m\int_0^{2\pi}\cos^{m+n+1}t\;\sin^{n}t\,dt
  \\&=\begin{cases}0&\text{$n$ odd or $m$ odd}\\2^{n+1}\int_0^{2\pi}\cos^{m+n+1}t\;\sin^{n}t\,dt&\text{otherwise}\end{cases}.
  \end{align*}
  Therefore these two curves have the same
  values on terms of the form \eqref{eq:expression}.
  \footnote{
  Note that they are not not tree-equivalent and therefore have different (iterated-integral) signatures. The lowest level on which they differ is level 4.}
\end{proof}
Moreover,
the algorithmic nature of the construction in \cite{bib:FKK2010}
makes it difficult to proceed to invariants of higher order.
In contrast, our method gives an explicit linear basis
for the invariants under consideration up to \emph{any} order.

Regarding the question of whether our invariants are complete we propose the following conjecture.
As shown in \cite{bib:HL2010}, if $S(X)_{0,T} = S(Y)_{0,T}$
for some curves $X,Y$, then $X$ is ``tree-like equivalent'' to $Y$.
For the concrete definition of this equivalence we refer to their paper,
but let us give one example.
Consider in $\ds=2$, the constant path $X_t := (0,0), t \in [0,T]$
and the piecewise linear path $Y$, between the points $(0,0),(1,0)$ and $(0,0)$.
One can check that
\begin{align*}
  S(X)_{0,T} = S(Y)_{0,T} = 1.
\end{align*}
The signature has no chance of picking up these kind of ``excursions'' in a path;
this concept is formalized in ``tree-like equivalence''.
We suspect that the following holds true
(with corresponding formulations for the other subgroups of $GL(\R^\ds)$).
\begin{conjecture}
  Let $X, Y: [0,T] \to \R^\ds$ be two curves such that
  \begin{align*}
    \Big\langle S(X)_{0,T}, \phi \Big\rangle = \Big\langle S(Y)_{0,T}, \phi \Big\rangle,
  \end{align*}
  for all $SO$ invariants given in Proposition~\ref{prop:soInvariants}.
  Then, there is a curve $\bar X$, tree-like equivalent to $X$,
  and a rotation $A \in SO(\R^d)$, such that
  \begin{align*}
    A \bar X = Y.
  \end{align*}
\end{conjecture}

~\\

Lastly, in Proposition~\ref{prop:glInvariants}, Proposition~\ref{prop:soInvariants} and Proposition~\ref{prop:permInvariants}
we have established a linear basis for invariants for every homogeneity.
As already mentioned in Remark~\ref{rem:algebraicIndependence},
owing to the shuffle identity,
there are algebraic relations between elements of different homogeneity.
An interesting open problem is then to find a minimal set of generators
for the set of invariants, considered as a subalgebra of the shuffle algebra.
(This applies to all subgroups of $GL(\R^\ds)$ and their corresponding invariants).

%\todonotes{ free generating set for concatentation algebra is straightforard. for $d=2$ use Dyck words.  compare also: Teranishi - Noncommutative classical invariant theory }

\end{document}